\documentclass{article}

\PassOptionsToPackage{numbers,compress}{natbib}
\usepackage[main,preprint]{neurips_2026}

\usepackage[utf8]{inputenc}
\usepackage[T1]{fontenc}
\usepackage{url}
\usepackage{booktabs}
\usepackage{amsfonts}
\usepackage{nicefrac}
\usepackage{microtype}
\usepackage{xcolor}
\usepackage{tikz}
\usetikzlibrary{positioning,arrows.meta}

\usepackage{amsmath,amssymb,amsthm,mathtools}
\usepackage{thmtools}
\usepackage{thm-restate}
\usepackage{bm}
\usepackage{graphicx}
\usepackage{subcaption}  
\usepackage{enumitem}
\usepackage{algorithm}
\usepackage{algpseudocode}   

\usepackage{multirow}
\usepackage{colortbl}
\usepackage{hyperref}
\usepackage{cleveref}

\newtheorem{claim}{Claim}
\newtheorem{theorem}{Theorem}[section]
\newtheorem{lemma}[theorem]{Lemma}
\newtheorem{definition}{Definition}

\theoremstyle{remark}
\newtheorem{remark}{Remark}

\title{Recovery Guarantees for Posterior Sampling of One-Bit Compressed Sensing}

\author{
Jing Ma$^{\star}$
\qquad
Yujia Wu$^{\star}$
\qquad
Zhaoqiang Liu$^{\dagger}$
\\
University of Electronic Science and Technology of China
}

\begin{document}

\maketitle

\begingroup
\renewcommand\thefootnote{}
\footnotetext{
$^{\star}$Equal contribution.
\quad
$^{\dagger}$Corresponding author (zqliu12@gmail.com). 
}
\endgroup

\maketitle

\begin{abstract}
We study the sample complexity of noisy one-bit compressed sensing for signals drawn from a prior distribution. By characterizing the effective distributional complexity of the prior via its approximate covering number, we prove that posterior sampling achieves accurate recovery with high probability when the number of measurements scales with the logarithm of the approximate covering number, up to a one-bit separation gap factor. This upper bound is robust to learned prior mismatch. Specifically, we show that posterior sampling with an approximate prior remains reliable, provided that the learned prior distribution is sufficiently close to the true signal distribution in Wasserstein distance. In addition, we establish a sample complexity lower bound for any reliable method of noisy one-bit compressed sensing, showing that our upper bound is nearly matched in its main prior dependent term. To approximate the ideal posterior sampling process for real world scenarios, we instantiate posterior sampling through a plug-and-play algorithm with diffusion priors. Experiments on the FFHQ and ImageNet datasets demonstrate the effectiveness of our proposed approach.

\end{abstract}

\section{Introduction}
Compressed sensing studies how to reliably recover high dimensional signals from a number of measurements far below the ambient dimension. Its success relies on a central premise: High dimensional signals often have low intrinsic dimension~\citep{baraniuk2009random,bora2017compressed,jalal2021instance}. Classical compressed sensing theory formalizes this structure through models such as sparsity~\citep{candes2006robust}, unions of subspaces~\citep{eldar2009robust}, low rankness~\citep{recht2010guaranteed}, or low dimensional manifolds~\citep{baraniuk2009random}, and establishes recovery guarantees when the number of
measurements is sufficiently large relative to the intrinsic complexity of the
underlying signal class~\citep{candes2006robust,donoho2006compressed,baraniuk2009random,recht2010guaranteed}. However, in many modern sensing systems, the measurement process is often further constrained by low precision acquisition devices, limited storage, or restricted communication bandwidth~\citep{jacques2013robust,laska2011trust,ruan2024adaptive,ruan2024optimal,han2024mmsign}. In such scenarios, the sensor may record only the signs of linear measurements, leading to the one-bit compressed sensing model~\citep{boufounos20081,plan2012robust}:
\begin{equation}
    y=\operatorname{sign}(Ax^\ast+\xi)\in\{-1,+1\}^m,
    \label{eq: compressedSensingModel}
\end{equation}
where \(x^\ast\in \mathbb{S}^{n-1}:=\{x\in\mathbb{R}^n:\|x\|_2=1\}\) is the unknown signal,
\(A\in \mathbb R^{m\times n}\) is the measurement matrix, and 
\(\xi\sim\mathcal N(0,\sigma^2I_m)\) denotes Gaussian noise with noise level \(\sigma>0\).
For one-bit measurements, there is a loss of the magnitude information on the signal~\citep{boufounos20081,plan2013one,knudson2016one} and we focus on directional recovery and measure reconstruction error
by geodesic distance on \(\mathbb S^{n-1}\).

Classical compressed sensing often adopts a set based viewpoint, seeking recovery over a prescribed structured class such as sparse vectors~\citep{needell2009uniform,needell2009cosamp}, low dimensional manifolds~\citep{hegde2007random,eftekhari2015new}, or range of generative model~\citep{bora2017compressed,hand2018phase}. While this perspective has been successful, it can be overly pessimistic for expressive priors since it ignores how probability mass is distributed within the signal model~\citep{jalal2021instance,reeves2012sampling}. A distributional viewpoint instead treats the signal prior itself as the object of study and asks for high probability recovery when the signal \(x^\ast\) is drawn from that prior~\citep{ji2008bayesian,yu2011statistical,reeves2012sampling,aeron2010information}. In this setting, it is more natural to measure the complexity of the distribution through the approximate covering number of a high probability region~\citep{jalal2021instance,adcock2025many}, rather than through the covering number of the entire support. For linear Gaussian measurements, it is shown that the approximate covering number governs the sample complexity of posterior sampling for arbitrary data priors~\citep{jalal2021instance}. This motivates the central question of this paper: For signals drawn from an arbitrary data prior distribution on the unit sphere, how should the sample complexity of reliable recovery from noisy one-bit measurements be characterized in terms of the prior distribution?

\subsection{Related works}
Early work on one-bit compressed sensing studied sensing systems in which only the signs of linear measurements are retained~\citep{boufounos20081}.
Since such sign observations discard amplitude information, traditional one-bit compressed sensing focuses on directional recovery, typically under a normalization constraint on the signal~\citep{plan2012robust,jacques2013robust}.
For sparse signals, foundational results establish recovery guarantees from noisy~\citep{kafle2022noisy} or noiseless sign measurements using convex programs and related optimization methods~\citep{plan2012robust,plan2013one,jacques2013robust,ai2014one}.
A key geometric principle underlying these results is that random hyperplanes can provide binary embeddings: As first proven in ~\citep{jacques2013robust} and soon after in ~\citep{plan2013one}, the normalized Hamming distance between sign patterns concentrates around the angular distance between signals, uniformly over low complexity sets.
Subsequent works further develop robust recovery guarantees under measurement noise, sign flips, and other quantization effects~\citep{jacques2013robust,knudson2016one}.
These classical results show that one-bit measurements can preserve sufficient angular information for reliable recovery over structured signal classes, despite the severe information loss caused by quantization.

With the development of deep learning, the seminal work~\citep{bora2017compressed} shows that signals lying near the range of a generative model can be recovered from significantly fewer linear measurements than the ambient dimension by optimizing over the latent space of the generator. This viewpoint replaces hand crafted structures, such as sparsity or low rankness, with the range of a learned generator, and has motivated a broad line of work on compressed sensing with generative priors~\citep{liu2020information,chen2023unified,chen2025solving}. Motivated by these developments, recent studies have extended generative prior methods to one-bit compressed sensing, where recovery is typically performed by searching for a latent code whose generated signal is consistent with the observed signs~\citep{joseph2020one,qiu2020robust,liu2020sample}.
Theoretical analyses in this direction often rely on binary stable embeddings or uniform concentration over the generator range, leading to sample complexity bounds for Lipschitz generative models under Gaussian hyperplane measurements~\citep{liu2020sample,liu2020information,qiu2020robust}. 
More recently, pretrained diffusion generative models~\citep{ho2020denoising, song2020ddim, songscore} have emerged as powerful deep generators, and are widely leveraged as highly flexible priors for signal recovery~\citep{chung2023diffusion,kawar2022ddrm,wang2022ddnm,wang2024dmplug,zheng2025integrating,wu2024principled,zhang2025improving,alkhouri2025sitcom,zheng2026image,chung2024decomposed,xu2024provably,coeurdoux2024plug,bouman2023generative,chang2025provable,jiao2026provable,liu2026posterior}, including subsequent extensions to one-bit compressed sensing
\citep{mengquantized, tang2025learning, chen2026diffusion}. However, the theoretical guarantees for these diffusion based methods remain relatively underexplored.
Unlike previous set based guarantees over prescribed signal classes or range of generator~\citep{liu2020sample,liu2022non,chen2023unified,chen2024efficient}, we study arbitrary data prior distributions on the unit sphere and characterize the sample complexity through the approximate covering complexity~\citep{jalal2021instance} of high probability regions.

% To operationalize such distributional viewpoints, extensive works leverage pre-trained diffusion models~\citep{ho2020denoising, song2020ddim, songscore} as flexible priors for signal recovery~\citep{chung2023diffusion,kawar2022ddrm,wang2022ddnm,bora2017compressed,wang2024dmplug,wu2024principled,zhang2025improving,alkhouri2025sitcom,chung2024decomposed,xu2024provably}. Recently, this generative paradigm has been extended to one-bit compressed sensing~\citep{mengquantized, tang2025learning, chen2026diffusion}. However, these existing approaches predominantly seek a single optimal point estimate via optimization techniques. To rigorously align our empirical implementation with the theoretical guarantees of posterior sampling, we propose a specialized one-bit plug-and-play algorithm.

 The work~\citep{jalal2021instance} is perhaps the most closely relevant work, which develops a posterior sampling framework for linear compressed sensing with arbitrary data priors. Their theory introduces approximate covering numbers as a central complexity measure for distributional compressed sensing and shows that posterior sampling is instance optimal under linear Gaussian measurements. We extend this idea to one-bit measurements. This extension is not a simple technical modification. Since one-bit observations retain only binary sign patterns, residual comparisons and likelihood separation arguments from continuous linear measurements are no longer directly applicable, and we can only distinguish spherical signals using a separation gap defined through the expected Hamming distance. Moreover, the nonlinearity of the one-bit measurement model makes the analysis of noise and model mismatch more difficult. Additive Gaussian noise, after passing through the sign function, no longer behaves as a simple continuous perturbation. When the true prior and the sampling prior are mismatched, we use a R\'enyi type stability~\citep{van2014renyi} argument to control how this perturbation affects the posterior distribution. In addition, the analysis is carried out on the unit sphere with geodesic distance, which requires spherical versions of covering, separation, and Wasserstein coupling arguments.

\subsection{Contributions}
Our work extends the distributional posterior sampling framework from continuous linear measurements to the noisy one-bit setting, together with the corresponding analysis of noisy binary separation, recovery guarantees under prior mismatch, and a nearly matching information-theoretic necessary condition in the leading prior complexity term. Our main results are summarized as follows:
\begin{itemize}
    \item 
    In~\Cref{upper_bound}, we derive an upper bound on the sample complexity with Gaussian measurement matrices. The result gives the upper bound of sample complexity conditions for achieving geodesic recovery accuracy \(\eta\), and is robust to mismatch between the learned prior distribution and the true signal distribution when they are sufficiently close in Wasserstein distance.

    \item 
    In~\Cref{lower_bound}, we derive a nearly matching lower bound on the sample complexity required by any reliable recovery method, showing that the approximate covering number in the upper bound is unavoidable.
    
    \item 
    In~\Cref{PnP}, we develop a specialized Plug-and-Play method with diffusion priors to approximate the posterior sampling process for noisy one-bit compressed sensing. The experimental results verify the effectiveness of this approach in image recovery tasks.
\end{itemize}
\section{Preliminaries}

\begin{definition}[Signal and observation distances]
Since signals are normalized to the unit sphere
\(\mathbb S^{n-1}\),
we measure signal distance by the normalized geodesic distance
$\mathrm{d}_{\mathrm{S}}(x_1,x_2):=\frac{1}{\pi}\arccos\langle x_1,x_2\rangle$, where $x_1,x_2\in\mathbb S^{n-1}$.
For binary observations \(u,v\in\{-1,+1\}^m\), we measure their discrepancy by the normalized Hamming distance \(\mathrm{d}_{\mathrm{H}}(u,v):=\frac{1}{m}\sum_{i=1}^m \mathbf 1\{u_i\neq v_i\}\).
\end{definition}
\begin{definition}[$(\eta,\delta)$-approximate covering number~\citep{jalal2021instance}]
For \(x_0\in\mathbb S^{n-1}\) and \(0<\eta\le 1\), define the geodesic $\eta$-cap
$\mathcal B(x_0,\eta):=\{x\in\mathbb S^{n-1}: \mathrm{d}_{\mathrm{S}}(x,x_0)\le \eta\}$. For a distribution \(R\) on \(\mathbb S^{n-1}\) and \(\delta\in[0,1)\), the
\((\eta,\delta)\)-approximate covering number is
\begin{equation}
\mathrm{Cov}_{\eta,\delta}(R):=\min\left\{k\in \mathbb N:R\left(\bigcup_{i=1}^k \mathcal B(x_i,\eta)\right)\ge 1-\delta,
\ x_i\in\mathbb S^{n-1}
\right\}.
\end{equation}
\end{definition}

\begin{definition}[Geodesic Wasserstein distance]
Let \(\mu\) and \(\nu\) be probability distributions on \(\mathbb S^{n-1}\), and let
\(\Pi(\mu,\nu)\) denote the set of all couplings of \(\mu\) and \(\nu\).
For \(p\ge 1\), the geodesic Wasserstein-\(p\) distance between \(\mu\) and \(\nu\) is defined as
\begin{equation}
\mathcal W_{p,\mathrm{geo}}(\mu,\nu)
:=
\left(
\inf_{\gamma\in\Pi(\mu,\nu)}
\int_{\mathbb S^{n-1}\times\mathbb S^{n-1}}
\mathrm{d}_{\mathrm{S}}(x,y)^p\, d\gamma(x,y)
\right)^{1/p}.
\end{equation}
The corresponding geodesic Wasserstein-\(\infty\) distance is defined as
\begin{equation}
\mathcal W_{\infty,\mathrm{geo}}(\mu,\nu)
:=
\inf_{\gamma\in\Pi(\mu,\nu)}
\operatorname*{ess\,sup}_{(x,y)\sim\gamma}
\mathrm{d}_{\mathrm{S}}(x,y).
\end{equation}
\end{definition}
\begin{definition}[One-bit separation gap]
\label{separation_gap}
For the noisy Gaussian one-bit model with noise level \(\sigma\), let \(f_\sigma(w)=\frac{1}{\pi}\arccos\!\left(\frac{\cos(\pi w)}{\sqrt{1+\sigma^2}}\right)\), \(w\in[0,1]\). For \(u\in(0,1)\) and \(s\in(1,1/u)\), the one-bit separation gap is defined by \(\Delta_\sigma(s,u):=f_\sigma(su)-f_\sigma(u)\). Since \(s>1\), we have \(su>u\), by the monotonicity of
\(f_\sigma\), \(\Delta_\sigma(s,u)\ge 0\).
\end{definition}
\begin{definition}[Maximum row $\ell_2$ norm]
Let $A\in\mathbb{R}^{m\times n}$ and let $a_i^\top$ denote the $i$-th row of $A$. We define
$\|A\|_{2,\infty}:=\max_{i\in[m]}\|a_i\|_2.$ 
\end{definition}
\begin{definition}[One-bit probit conditional likelihood]
Under the one-bit probit model, the conditional likelihood of observing 
$y\in\{-1,+1\}^m$ given $x$ and $A$ is defined as
$\Pr(y| x,A):=\prod_{i=1}^m\Phi\left(\frac{y_i a_i^\top x}{\sigma}\right)$, where \(\Phi\) is the standard Gaussian cumulative distribution function and \(\sigma>0\) is the noise level.
\end{definition}

\section{Main results}
In this section, we study the sample complexity of reliable signal recovery by presenting both upper bound and lower bound. Specifically, the upper bound provides a sufficient measurement condition for reliable recovery, whereas the lower bound provides a necessary measurement condition that must be satisfied by any reliable recovery method. In the upper bound analysis, we assume an i.i.d. Gaussian sensing matrix \(A_{ij}\sim \mathcal N(0,1)\), which ensures that geodesic separation on the sphere is translated into predictable Hamming separation between one-bit sign patterns. 
In the lower bound analysis, we consider both deterministic and Gaussian sensing matrices.

\subsection{Upper bound}
\label{upper_bound}
In this section, we establish an upper bound on the sample complexity of posterior sampling for one-bit compressed sensing. The key idea is that, under sufficiently many independent Gaussian measurements, geodesically separated signals induce one-bit observations that can be distinguished through their Hamming disagreement.

\begin{theorem}[One-bit posterior sampling upper bound]%{theorem}{uppercorollary}
\label{cor:main}
Let \(R\) and \(P\) be two probability distributions on the unit sphere \(\mathbb S^{n-1}\).\footnote{\(R\) is regarded as the true signal distribution and \(P\) is regarded as learned distribution corresponding to the generative model.} Let \(p\ge 1\), \(0<\eta<1\), \(0<\delta<1/4\), and assume 
$\mathcal W_{p,\mathrm{geo}}(R,P)\le \rho $. Define the effective mismatch level
$\varepsilon_\delta:=\rho/\delta^{1/p}$,
and define the parameter $
\bar\eta
:=
\eta+\varepsilon_\delta $.
We take $c\in\left(1,\frac{1-\varepsilon_\delta}{\bar\eta}-1\right)$, assume further that \(2\eta+3\varepsilon_\delta<1\), which guarantees that
$\left(1,\frac{1-\varepsilon_\delta}{\bar\eta}-1\right)$
is a nonempty interval. Given \((y,A)\),
$\widehat x
\sim
P(\cdot| y,A)$
is the posterior sampling estimator based on \(P\).
Furthermore, define
\(r_{\varepsilon_\delta}
:=
2\sin\left(
\frac{\pi\varepsilon_\delta}{2}
\right)\), assume that the model mismatch is sufficiently small, such that
\(\frac{r_{\varepsilon_\delta}^2}{\sigma^2}
\le\frac{1}{6}\Delta_\sigma(c,\bar\eta)^2\).
For \(\zeta\in(0,1)\), when
\(m\ge C\frac{\log \mathrm{Cov}_{\eta,\delta}(R)+\log 2+\log(1/\zeta)}
{\Delta_\sigma(c,\bar\eta)^2}\) for some \(C>0\), with probability at least \(1-\zeta\) over the draw of \(A\),
\begin{equation}
    \Pr_{x^*\sim R,\xi,\widehat{x}| A}\left(
\mathrm{d}_{\mathrm{S}}(x^*,\widehat x)
\ge(c+1)\bar\eta+\varepsilon_\delta
\right)
\le2\delta+\zeta .
\end{equation}
In particular, if \(c>1\) is also a fixed constant and
\(\varepsilon_\delta=O(\eta)\), then \(\bar\eta=\Theta(\eta)\).
In this case, to ensure
\(\mathrm{d}_{\mathrm{S}}(x^*,\widehat x)=O(\eta)\) with probability at least
\(1-2\delta-\zeta\) conditionally on a measurement matrix \(A\), it suffices to take
\begin{equation}
    m=
\begin{cases}
O\!\left((\log \mathrm{Cov}_{\eta,\delta}(R)+\log(1/\zeta))\,\eta^{-2}\right),
& \sigma\lesssim \eta,\\[2mm]
O\!\left((\log \mathrm{Cov}_{\eta,\delta}(R)+\log(1/\zeta))\,\sigma^2\eta^{-4}\right),
& \sigma\gtrsim \eta .
\end{cases}
\end{equation}
\end{theorem}

\begin{remark}
The condition $\frac{r_{\varepsilon_\delta}^2}{\sigma^2}\le\frac{1}{6}\Delta_\sigma(c,\bar\eta)^2$
requires the model mismatch to be small compared with the effective noise level and the one-bit near--far separation gap. Since $
r_{\varepsilon_\delta}=2\sin\left(\frac{\pi\varepsilon_\delta}{2}\right),\varepsilon_\delta=\frac{\rho}{\delta^{1/p}},$
for small \(\varepsilon_\delta\) this condition is equivalent up to constants to $\rho\lesssim\delta^{1/p}\sigma\Delta_\sigma(c,\bar\eta)$.
Thus, the learned sampling prior \(P\) must be sufficiently close to the true prior \(R\). If the mismatch is too large, the main probability masses of \(R\) and \(P\) may be separated on the sphere, so even after Gaussian smoothing their induced one-bit observation distributions may not overlap enough. Consequently, the posterior induced by \(P\) may place little mass near the true signal \(x^*\sim R\), causing the stability guarantee for posterior sampling to fail. Moreover, extensive studies have shown the generated distribution \(P\) can be close to the true distribution \(R\) under appropriate conditions~\citep{chen2023sampling,lee2023convergence,gao2025wasserstein,lu2020universal,gao2023approximating}.

\end{remark}

\begin{remark}
Compared with the linear posterior sampling upper bound~\citep{jalal2021instance}, the one-bit setting contains an additional factor
\(\Delta_\sigma(c,\bar\eta)^{-2}\). This is due to the one-bit observation process, which preserves only sign information and discards the amplitude information available in linear measurements. Consequently, the geometric separation between two covering caps is not directly reflected in the observation space, but can only be detected through the difference between their induced expected Hamming disagreements. This difference is precisely
$\Delta_\sigma(c,\bar\eta)=f_\sigma(c\bar\eta)-f_\sigma(\bar\eta)$.
To reliably distinguish the covering caps, the empirical Hamming distance must concentrate at the scale of \(\Delta_\sigma(c,\bar\eta)\), which requires
$m\Delta_\sigma(c,\bar\eta)^2\gtrsim \log \mathrm{Cov}_{\eta,\delta}(R)$.
Therefore, the sample complexity naturally contains the factor
$\Delta_\sigma(c,\bar\eta)^{-2}$.
This term reflects the intrinsic loss of information caused by one-bit quantization and noise.
\end{remark}

We first outline the proof of \Cref{cor:main}. Since one-bit observations retain only sign information, the Euclidean residuals used in linear models are no longer directly available. Therefore, we need to establish a new notion of separability in the binary observation space. Specifically, we first characterize the relationship between spherical geodesic distance and the Hamming disagreement rate of one-bit observations, and then use this relationship to construct a near--far test. We then combine approximate covering with a union bound to control the probability that the posterior sample falls into an incorrect far-away region, and finally use a R\'enyi type stability to handle model mismatch.

We first need to characterize how distinguishable two spherical signals are after one-bit observation. Due to the rotational invariance of Gaussian measurements, the sign disagreement probability of two signals depends only on their geodesic distance. The following lemma gives an explicit expression for this relationship.
\begin{lemma}[Expected Hamming distance between noisy and noiseless one-bit measurements]
    \label{hammingdistance}
    Consider two signals $x_1,x_2\in\mathbb S^{n-1}$ defined on the unit Euclidean sphere. Define the observations of these two signals as $u=\operatorname{sign}(Ax_1+\xi)$ and $v=\operatorname{sign}(Ax_2)$. We have:
\begin{equation}
\mathbb{E}[\mathrm{d}_{\mathrm{H}}(u,v)] = \frac{1}{\pi} \arccos\left(\frac{\cos(\pi \mathrm{d}_{\mathrm{S}}(x_1,x_2))}{\sqrt{1+\sigma^2}}\right).
\end{equation}
\end{lemma}

The previous lemma shows that geodesic separation is transformed by the one-bit channel into a gap in expected Hamming disagreement. Hence the quantity
$\Delta_\sigma(c,\eta)=f_\sigma(c\eta)-f_\sigma(\eta)$ measures the observational margin between signals inside $\mathcal B(x_2,\eta)$ and signals outside \(\mathcal B(x_2,c\eta)\)-cap. This margin allows us to construct the following near--far test.

\begin{lemma}[One-bit near--far test lemma]
\label{lem:near_far_test}
    Fix a reference point \(x_0\in\mathbb S^{n-1}\). Let the testing threshold be $\tau_{\eta,c,\sigma}:=\frac{f_\sigma(\eta)+f_\sigma(c\eta)}{2}$,
and define the test function
$\phi_{x_0,\eta}(y;A):=
\mathbf 1\left\{\mathrm{d}_{\mathrm{H}}\left(y,\operatorname{sign}(Ax_0)\right)\ge\tau_{\eta,c,\sigma}\right\}$.
Then the following two testing error bounds hold in the annealed sense:
\begin{equation}
    \sup_{x:\ \mathrm{d}_{\mathrm{S}}(x,x_0)\le \eta}
\Pr_{A,\xi| x}
\left(\phi_{x_0,\eta}(y;A)=1\right)
\le
\exp\left(-\frac{m}{2}\Delta_\sigma(c,\eta)^2\right),
\end{equation}

and
\begin{equation}
\sup_{x:\ \mathrm{d}_{\mathrm{S}}(x,x_0)\ge c\eta}\Pr_{A,\xi| x}
\left(\phi_{x_0,\eta}(y;A)=0\right)
\le
\exp\left(-\frac{m}{2}\Delta_\sigma(c,\eta)^2\right).
\end{equation}
Here the probability is taken with respect to the joint randomness of the measurement matrix \(A\) and the noise \(\xi\), conditional on the fixed reference point \(x\). Hence the result is annealed, in the sense that the probability is averaged over the measurement matrix \(A\), rather than a conditional in \(A\) bound for a fixed realization of \(A\).
\end{lemma}

\begin{remark}
~\Cref{hammingdistance} and ~\Cref{lem:near_far_test} are closely related to the distance-preservation principle underlying binary stable embeddings. When $\sigma=0$, we have $f_0(w)=w$, so ~\Cref{hammingdistance} reduces to the classical random-hyperplane identity $\mathbb{E}[\mathrm{d}_{\mathrm{H}}(\operatorname{sign}(Ax_1),\operatorname{sign}(Ax_2))]=\mathrm{d}_{\mathrm{S}}(x_1,x_2)$, with $\Delta_0(c,\eta)=(c-1)\eta$, while ~\Cref{lem:near_far_test} yields concentration of the empirical Hamming distance and hence a near--far separation guarantee. Unlike standard binary stable embeddings, our posterior sampling analysis only requires such near--far separation over an approximate cover of a high-probability region of the prior, rather than uniform distance preservation over all signal pairs. For $\sigma>0$, $f_\sigma$($\cdot$) and $\Delta_\sigma(c,\eta)$ quantify the additional distortion induced by Gaussian noise before quantization. A visual explanation of this section can be found in the diagram below, the shaded area represents the cases in which we recover from errors.
\end{remark}
\begin{figure}[H]
    \centering
    \includegraphics[width=\textwidth]{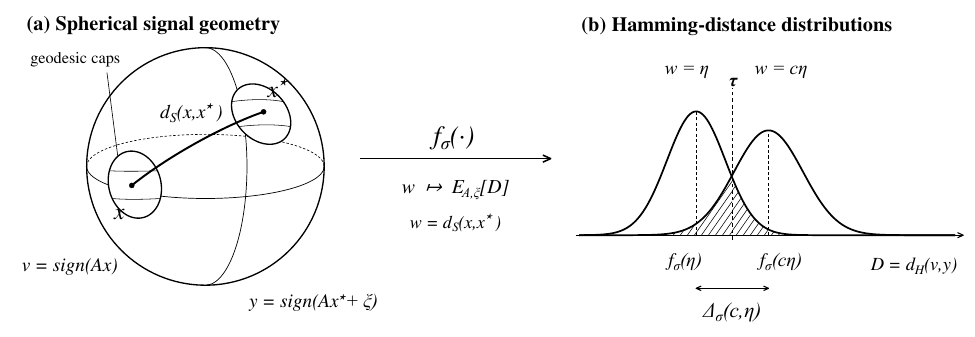}
\end{figure}

We now extend the near--far test to a general prior by covering its high probability region with finitely many \(\eta\)-caps. 
Applying the same far region control to each cap and taking a union bound over the covering centers yields the \(\log \mathrm{Cov}_{\eta,\delta}\) term in the sample complexity.

\begin{lemma}[Overall recovery upper bound for one-bit posterior sampling under a covering partition]
\label{lemma:orup}
Let \(Q\) be a probability distribution on \(\mathbb S^{n-1}\). Assume that there exist centers \(x_1,\ldots,x_N\in\mathbb S^{n-1}\) such that, with $ \mathcal C_j:=\{x\in\mathbb S^{n-1}: \mathrm{d}_{\mathrm{S}}(x,x_j)\le \eta\}, G:=\bigcup_{j=1}^N\mathcal C_j. $
Fix \(\alpha\in(0,\frac{1}{2})\), and assume that
\begin{equation}
    Q=(1-\alpha)Q_g+\alpha Q_b \qquad Q_g(G)=1, 
\end{equation}

Here \(Q_g\) denotes a covered good component of \(Q\), of mixture weight
\(1-\alpha\), contained in the high-probability region \(G\). Let $x^*\sim Q_g$, and let
$\widehat x\sim Q(\cdot| y,A)$
be the posterior sampling estimator based on the same distribution \(Q\). Suppose \(0<\eta<\frac{1}{2}\), and \(c\in(1,\frac{1}{\eta}-1)\). Then we have:
\begin{equation}
\Pr_{x^*\sim Q_g,A,\xi,\widehat{x}}\left(\mathrm{d}_{\mathrm{S}}(x^*,\widehat x)\ge(c+1)\eta\right)
\le
\frac{1}{1-\alpha}N\exp\left(-\frac{m}{2}\Delta_\sigma(c,\eta)^2\right).
\end{equation}
\end{lemma}

The previous lemma handles the matched prior finite cover case. To extend it to the general setting with distribution mismatch, the following lemma controls the probability loss under $\mathcal W_{\infty,\mathrm{geo}}$ mismatch through a R\'enyi type stability argument~\citep{van2014renyi}.

\begin{lemma}[R\'enyi stability under spherical $W_{\infty,\mathrm{geo}}$ mismatch]
\label{lem:renyi_onebit_mismatch_explicit}
Let \(R_0\) and \(S_0\) be two probability distributions defined on the unit sphere \(\mathbb S^{n-1}\). Assume that \(R_0\) and \(S_0\) satisfy the spherical Wasserstein-\(\infty\) mismatch condition $\mathcal W_{\infty,\mathrm{geo}}(R_0,S_0)\le \epsilon,$ where \(0<\epsilon\le 1\). Let \(D\) be a probability distribution on \(\mathbb S^{n-1}\). Let \(x^*\sim R_0, z^*\sim S_0\). Let \(\Pi\) be a coupling of \(R_0\) and \(S_0\) such that, for
\((x^*,z^*)\sim\Pi\), \(\mathrm{d}_{\mathrm{S}}(x^*,z^*)\le \epsilon\) almost surely. Let \(\widehat x\) and \(\widehat z\) be posterior samples with respect to
the model prior \(D\), conditioned on \((A,y)\) and \((A,u)\). Consider the noisy one-bit measurement model
\begin{equation}
    y=\operatorname{sign}(Ax^*+\xi),
    \qquad
    u=\operatorname{sign}(Az^*+\xi'),
\end{equation}
where \(A_{ij}\sim\mathcal N(0,1)\), \(\xi,\xi'\sim \mathcal N(0,\sigma^2 I_m)\), and \(A\), \(\xi\), and \(\xi'\)
are mutually independent.
Let $r_\epsilon=2\sin\left(\frac{\pi\epsilon}{2}\right)$.
Then, for any $d>0$ and any parameter $
0<\lambda\le
\min\left\{
1,\,
\frac{\sigma^2}{3r_{\epsilon}^2}
\right\}
$, we have
\begin{equation}
\label{eq:explicit_renyi_mismatch_bound}
    \Pr\left(
    \mathrm{d}_{\mathrm{S}}(x^*,\widehat x)\ge d+\epsilon
    \right)
    \le
    \left(
    1-\frac{\lambda(1+\lambda)r_\epsilon^2}{\sigma^2}
    \right)^{-\frac{m}{2(1+\lambda)}}
    \Pr\left(
    \mathrm{d}_{\mathrm{S}}(z^*,\widehat z)\ge d
    \right)^{\frac{\lambda}{1+\lambda}}.
\end{equation}
\end{lemma}

The proof of \Cref{cor:main} is obtained by combining
\Cref{lemma:orup,lem:cov_lemma,lem:renyi_onebit_mismatch_explicit}.
First, \Cref{lem:cov_lemma} reduces the general mismatched priors
\((R,P)\) to their high probability components \((R',P')\), losing only
\(2\delta\) probability mass. Since \(R'\) and \(P'\) are coupled within
geodesic distance
\(\varepsilon_\delta=\rho/\delta^{1/p}\), the covering radius is enlarged
from \(\eta\) to
\(\bar\eta=\eta+\varepsilon_\delta\).
Then \Cref{lemma:orup} gives the posterior sampling recovery bound for the
covered good component under the learned prior distribution \(P\).
Finally, \Cref{lem:renyi_onebit_mismatch_explicit} transfers this bound from
\(P'\) to \(R'\), adding the R\'enyi stability factor and an extra
\(\varepsilon_\delta\) in the recovery radius.
At this point, the error probability is controlled in expectation over the
random matrix \(A\). Hence, for any \(\beta\in(0,1)\), Markov's inequality
implies that the corresponding conditional error probability, given \(A\), is at
most \(1/\beta\) times this expectation with probability at least \(1-\beta\)
over \(A\).
Combining these terms and simplifying yields Lemma~\ref{lem:main}. Theorem~\ref{cor:main} then follows by specializing the auxiliary
parameters in Lemma~\ref{lem:main} and simplifying the resulting expression under the small-mismatch condition.

For the specific proof methods and processes, please refer to \Cref{upper_bound_proofs} in Appendix.
\subsection{Lower bound}
\label{lower_bound}

The one-bit observation model induces a stronger information bottleneck than the linear sensing model. In the Gaussian linear model, each real valued measurement may carry an amount of information that scales as a logarithmic signal-to-noise ratio. In contrast, under one-bit quantization, the observation $y\in\{-1,+1\}^m$ is a binary string, and therefore the total information transmitted by the sensing process is intrinsically limited by the entropy of $m\log 2$ nats. Consequently, the mutual information argument used for continuous compressed sensing must be replaced by a channel bound adapted to the one-bit model. This section formalizes this replacement and combines it with a spherical Fano argument to obtain a lower bound on the number of measurements. Throughout this section, logarithms are natural logarithms.

\begin{theorem}[Information lower bound for one-bit compressed sensing]%{theorem}{onebitspherethm}
\label{thm:one_bit_sphere_lower_bound}
Let $R$ be a prior distribution defined on the high dimensional unit Euclidean sphere $\mathbb{S}^{n-1}$, where $x^* \sim R$. Define the recovery accuracy \(\eta\in(0,\frac{1}{3})\), suppose that $\delta<0.1$ and that there exists a recovery algorithm capable of outputting $\widehat{x} \in \mathbb{S}^{n-1}$ based solely on the observation $y$ and the matrix $A$, such that:
\begin{equation}
    \Pr\left(d_{S}(\widehat{x}, x^*) \le \eta\right) \ge 1-\delta.
\end{equation}
If $A$ is a deterministic matrix, then the number of measurements $m$ must satisfy:
\begin{equation}
    m \ge \frac{0.072\pi\sigma^2}{2(\log 2)\|A\|_{2,\infty}^2} \left( \log \mathrm{Cov}_{3\eta, 4\delta}(R) + \log(6\delta)-25 \right).
\end{equation}
If $A_{ij}\sim \mathcal N(0,1)$, then the number of measurements $m$ must satisfy:
\begin{equation}
    m \ge \frac{0.072\pi}{2(\log 2)\arcsin\left(\frac{1}{1+\sigma^2}\right)} \left( \log \mathrm{Cov}_{3\eta, 4\delta}(R) + \log(6\delta)-25\right).
\end{equation}
\end{theorem}
\begin{remark}[Comparison between the upper and lower bounds]
We now compare the upper bound in \Cref{cor:main} with the lower bound in
\Cref{thm:one_bit_sphere_lower_bound}. Then \Cref{cor:main} shows that when
$m\gtrsim\frac{\log \mathrm{Cov}_{\eta,\delta}(R)}{\Delta_\sigma(c,\eta)^2}$, posterior sampling achieves accurate recovery with high probability. In contrast, \Cref{thm:one_bit_sphere_lower_bound} shows that any estimator
that achieves \(O(\eta)\) accurate recovery with constant success probability
must satisfy
$m\gtrsim\log \mathrm{Cov}_{3\eta,4\delta}(R)$,
where the constants may depend on the one-bit channel and the noise level.

Therefore, the two bounds nearly match in the leading term $\log \mathrm{Cov}_{\eta,\delta}(R)$ depending on the prior $R$.
This shows that the approximate covering number is the correct distributional
complexity measure for one-bit compressed sensing with arbitrary data priors.
In particular, the complexity term appearing in the posterior sampling upper
bound generally cannot be removed or substantially improved.

The remaining difference is mainly the factor
\(\Delta_\sigma(c,\eta)^{-2}\) in the upper bound. This factor comes from
the one-bit observation process. Since the measurements only retain the signs,
two signals separated on the sphere can be distinguished only through the gap
between their expected Hamming distances. This gap is exactly measured by
\(\Delta_\sigma(c,\eta)\). If $c$ is chosen as a fixed numerical constant, for instance $c=2$, then in the small-noise regime, we have $\Delta_\sigma(c,\eta) \asymp \eta$. Therefore, the upper bound requires sufficiently many measurements for this Hamming-distance gap to concentrate.
Thus, the upper and lower bounds are nearly matched in terms of \(\log \mathrm{Cov}_{\eta,\delta}(R)\). 
\end{remark}

\begin{remark}
    For the lower bound, we explicitly consider arbitrary fixed deterministic sensing matrices and i.i.d. Gaussian random matrices, which already cover many commonly used sensing settings. In particular, since the deterministic result holds for any fixed \(A\), it also allows \(A\) to be optimized using knowledge of the prior distribution \(R\), provided that \(A\) is independent of the realized signal \(x^\ast\). Our current analysis does not establish a universal lower bound for more general randomized or adaptive sensing strategies. In such settings, the sensing matrix may depend on additional randomness or side information correlated with the realized signal, in which case the independence argument used in ~\Cref{lemma:4.3} no longer applies directly and a separate information-theoretic analysis is required.
\end{remark}
Following the method in~\citep{jalal2021instance}, we establish an upper bound on the mutual information conveyed by the one-bit observation channel. We derive our first lemma by starting from the mutual information conveyed between the observation $y$ and the ground truth $x^*$.

\begin{lemma}[Mutual Information Bound]%{lemma}{oneupper}
\label{lemma:oneupper}
Assume that the measurement process is $y = \operatorname{sign}(Ax^*+\xi)$, where $\xi\sim\mathcal N(0,\sigma^2I_m)$. If $A$ is a deterministic matrix, we have
\begin{equation}
    I(y; x^*) \le \frac{2m\log 2}{\pi \sigma^2} \|A\|_{2,\infty}^2.
\end{equation}
If $A_{ij}\sim\mathcal N(0,1)$, we have
\begin{equation}
   I(y; x^* | A) \le \frac{2m \log 2}{\pi} \arcsin\left(\frac{1}{1+\sigma^2}\right).
\end{equation}
\end{lemma}
Next, utilizing the properties of Markov chains, we express the mutual information between $x^*$ and $\widehat{x}$. Here, we directly refer to the conclusion in~\citep{jalal2021instance} and obtain that if $A$ is a Gaussian matrix, we have $I(x^*; \widehat{x}) \le I(y; x^*|A)$; if $A$ is a fixed matrix, we have $I(x^*; \widehat{x}) \le I(y; x^*)$ (see Lemma~\ref{lemma:4.3} for details).

We next establish a spherical Fano type inequality, which converts the high probability recovery guarantee into a lower bound on the mutual information between the signal and its estimate.

\begin{lemma}[Variant of spherical Fano inequality]%{lemma}{spherefano}
\label{lemma:spherefano_new}
Let $(x, \widehat{x})$ be random variables with a joint distribution on the unit Euclidean sphere $\mathbb{S}^{n-1} \times \mathbb{S}^{n-1}$, where $x \sim R$, and suppose that the estimator satisfies
\begin{equation}
    \Pr\left(d_{S}(x, \widehat{x}) \le \eta\right) \ge 1-\delta.
\end{equation}
Assuming $\delta$ is sufficiently small, then for any $\tau \in (0, 1-2\delta)$, the following holds:
\begin{equation}
    (1-\delta)\tau(1-2\delta)\log \mathrm{Cov}_{3\eta, \tau+3\delta}(R) \le I(x; \widehat{x}) + 2(1-\delta).
\end{equation}
\end{lemma}
Now, based on Lemma~\ref{lemma:spherefano_new}, we obtain the requirement of the mutual information on the effective geometric complexity, that is, $(1-\delta)\tau(1-2\delta)\log \mathrm{Cov}_{3\eta, \tau+3\delta}(R) \le I(x; \widehat{x}) + 2(1-\delta)$. On the other hand, according to Lemma~\ref{lemma:oneupper} and Lemma~\ref{lemma:4.3} in Appendix, we derive the upper bound of the mutual information: When $A$ is a fixed matrix, we have $ I(\widehat{x}; x^*) \le \frac{2m\log 2}{\pi \sigma^2} \|A\|_{2,\infty}^2$ and when $A$ is a Gaussian matrix, we have $I(\widehat{x}; x^* | A) \le \frac{2m \log 2}{\pi} \arcsin\left(\frac{1}{1+\sigma^2}\right)$. Therefore, substituting the corresponding mutual information upper bound into the Fano inequality requirement and solving for \(m\) yields the desired lower bounds in Theorem~\ref{thm:one_bit_sphere_lower_bound}.

For the specific proof methods and processes, please refer to \Cref{section_proof_B} in Appendix.
\section{Algorithm and experiments}
\label{PnP}
\subsection{Plug-and-play approach for one-bit measurements}

% In one-bit quantized signal recovery, given the measurements $y$ and the measurement matrix $A$, our objective is to reconstruct the unknown signal $x^*$ in Eq.~\eqref{eq: compressedSensingModel} by sampling from the posterior distribution. The posterior distribution can be formulated as:
While our theory establishes guarantees for ideal posterior sampling, exact inference in high dimensions is intractable. To provide a practical counterpart, we introduce PnP-OneBit\footnote{The corresponding code is available at \href{https://github.com/Dugehude/pnp-onebit}{https://github.com/Dugehude/pnp-onebit}.}, which approximates this sampling process using diffusion models as the learned prior. Specifically, in one-bit quantized signal recovery, given measurements $y$ and the forward matrix $A$, our objective is to reconstruct $x^*$ in Eq.~\eqref{eq: compressedSensingModel} by sampling from the target posterior:
\begin{equation}
    p(x|y, A) \propto p(y|x, A) p(x) = \exp\left(-\mathcal{L}(x;y) - \mathcal{P}(x)\right),
\end{equation} 
where $\mathcal{P}(x) := -\log p_{\text{data}}(x)$ is the negative log-prior function, and $\mathcal{L}(x;y) := -\log p(y|x, A)$ is the negative log-likelihood function. In one-bit signal recovery, we can exactly express $\mathcal{L}(x;y)= -\sum_{i=1}^m \log \Phi\left(\frac{y_i a_i^\top x}{\sigma}\right)$ based on the cumulative distribution function of the standard normal distribution $\Phi(\cdot)$.

While existing one-bit recovery methods like Diff-OneBit~\cite{chen2026diffusion} typically seek a single optimal estimate via half quadratic splitting (HQS)~\cite{geman1995nonlinear}, we propose PnP-OneBit for more robust signal reconstruction. Earlier works have explored sampling-based plug-and-play approaches for inverse problems. In particular, the work~\citep{coeurdoux2024plug} incorporated diffusion-based generative priors into a split Gibbs sampling framework, while work~\citep{bouman2023generative} developed a related generative plug-and-play framework for posterior sampling. Building on these developments, we adapt the PnP-DM framework~\citep{xu2024provably} and the split Gibbs sampler (SGS) ~\citep{vono2019split} to the nonlinear one-bit likelihood. This alternating sampling strategy aligns with the algorithmic structure of \citep{xu2024provably}, which has been shown to be provably robust for image recovery. By introducing an auxiliary variable $z \in \mathbb{R}^n$ and a coupling parameter $\varrho > 0$, we construct a joint distribution
\begin{equation}
    \Pi(x,z) \propto \exp\left(-\mathcal{L}(z;y) - \mathcal{P}(x) - \frac{1}{2\varrho^2}\|x-z\|_2^2\right).
\end{equation}
Starting from an initialization $x^{(0)}$, SGS alternates between sampling from the following two conditional distributions at each iteration $k=1,2,\dots,K$:
\begin{equation}
    \label{eq:sample_z}
    z^{(k)} \sim \Pi(z|x^{(k-1)}) \propto \exp\left( -\mathcal{L}(z;y) - \frac{1}{2\varrho^2}\|z-x^{(k-1)}\|_2^2 \right),
\end{equation}
\begin{equation}
    \label{eq:sample_x}
    x^{(k)} \sim \Pi(x|z^{(k)}) \propto \exp\left( -\mathcal{P}(x) - \frac{1}{2\varrho^2}\|x-z^{(k)}\|_2^2 \right).
\end{equation}

\begin{algorithm}[H]
\caption{PnP-OneBit}
\label{pnp_onebit}
\begin{algorithmic}[1]
\Require iteration number $K$
\State Sample $x^{(0)} \sim \mathcal{N}({0}, I_n)$
\For{$k=0,1,\dots,K-1$}
    \State Sample $z^{(k)} \sim \Pi(z|x^{(k)})$ via Langevin Monte Carlo
    \State Sample $x^{(k+1)} \sim \Pi(x|z^{(k)})$ via DDIM denoising
\EndFor
\State \Return $x^{(K)}$

\end{algorithmic}
\end{algorithm}

Sampling from Eq.~\eqref{eq:sample_z} can be achieved via Langevin Monte Carlo due to the differentiability of $\mathcal{L}(z;y)$ in the one-bit setting. However, due to $\mathcal{P}(x)$, directly sampling from Eq.~\eqref{eq:sample_x} is intractable. The decoupling between data fidelity and prior allows us to utilize a pretrained diffusion model~\cite{song2020ddim,ho2020denoising} as a plug-and-play prior. The forward process of a diffusion model perturbs the data distribution to a noise distribution over continuous time $t \in [0, T]$, characterized by the conditional distribution $q(x_t|x_0)$ as:
\begin{equation}
\label{eq:forward}
    x_t|x_0\sim\mathcal{N}(x_t; \alpha_t x_0, \sigma^2_t I_n),
\end{equation}
where $x_0 \sim p_{\text{data}}$ denotes the clean data, and $\alpha_t$ and $\sigma_t$ dictate the noise schedule. This forward process can be rewritten as a stochastic differential equation (SDE)~\cite{songscore}:
\begin{equation}
\label{eq:forwardSDE}
    d x_t = f(t) x_t d t + g(t) d w_t,
\end{equation}
where $f(t)x_t$ and $g(t)$ are the drift and diffusion coefficients, and $w_t$ is the standard Wiener process. To match the marginal distribution of Eq.~\eqref{eq:forward} and Eq.~\eqref{eq:forwardSDE}, the drift and diffusion terms should satisfy~\citep{lu2022dpmsolver}:

\begin{equation}
    f(t) = \frac{d\log \alpha_t}{dt},\quad g^2(t) = \frac{d\sigma_t^2}{dt} - 2f(t)\sigma_t^2.
\end{equation}
To simulate the reverse generative process, a neural network $s_{\theta}(x_t, t)$ is pretrained via score matching~\cite{songscore} to approximate the intractable score function $s_\theta(x_t,t)\approx\nabla_{x_t} \log p_t(x_t)$ where $x_t\sim p_t$. The reverse process is then governed by SDE~\cite{songscore} as:
\begin{equation}
\label{eq:inverSDEMain}
    dx_t = \left(f(t)x_t - g^2(t)\nabla_{x_t} \log p_t(x_t)\right)dt + g(t)d\overline w_t,
\end{equation}
where $\overline w_t$ is the inverse-time standard Wiener process. Detailed derivations of the diffusion model are provided in \Cref{appendix: Preliminary4diffusion} in Appendix.

We establish that the prior sampling step in Eq.~\eqref{eq:sample_x} mathematically aligns with the diffusion posterior $p(x_0|x_t)$, when the diffusion time $t$ matches the condition $\varrho = \frac{\sigma_t}{\alpha_t}$ and $z^{(k)}=x_t/\alpha_t$ (the corresponding derivation is deferred to \Cref{appendix: AlignBetweenLikelihoodNPrior} in Appendix). Consequently, we can obtain an approximate sample $x^{(k)}$ by initializing the reverse SDE in Eq.~\eqref{eq:inverSDEMain} at $x_t = \alpha_t z^{(k)}$ and integrating it back to $t=0$. By alternating the likelihood sampling in Eq.~\eqref{eq:sample_z} and the diffusion-based prior sampling in Eq.~\eqref{eq:sample_x} over $K$ iterations, we ultimately obtain the reconstructed signal $x^{(K)}$, which approximates a sample from the one-bit posterior distribution.~\Cref{pnp_onebit} presents our algorithm in a summarized manner, The complete sampling algorithm is summarized in \Cref{appendix: DetailedAlgorithm} in Appendix.

\subsection{Quantitative results}
Our experiments are conducted on the FFHQ~\cite{karras2019style} and ImageNet~\cite{russakovsky2015imagenet} datasets at a spatial resolution of $256 \times 256$. Following standard practices in one-bit compressed sensing~\citep{mengquantized,tang2025learning,chen2026diffusion}, the elements of the measurement matrix $A$ are sampled from $\mathcal{N}(0, 1/m)$. We compare the proposed PnP-OneBit against several models, including DiffPIR~\cite{zhu2023denoising}, DPS~\cite{chung2023diffusion}, DAPS~\cite{zhang2025improving}, QCS-SGM~\cite{mengquantized}, SIM-DMIS~\cite{tang2025learning}, and Diff-OneBit~\cite{chen2026diffusion}. For a fair comparison, all evaluated algorithms utilize the same pretrained diffusion models and operate on identical measurements $y$. For each dataset, we randomly select 100 test images. We evaluate each method five times per image and report the mean $\pm$ standard deviation for peak signal-to-noise ratio (PSNR), structural similarity index measure (SSIM), and learned perceptual image patch similarity (LPIPS)~\cite{zhang2018unreasonable}. Further implementation details are provided in \Cref{appendix: ImplementationDetail} in Appendix.

\Cref{tab: QuantitativeResultsFFHQ} summarizes the quantitative results on the FFHQ dataset under different compression ratios $n/m$. Additional evaluations on the ImageNet dataset, experiments under varying noise levels, and comprehensive further studies are deferred to \Cref{appendix: ExperimentalResult} in Appendix. Across both datasets, the proposed PnP-OneBit consistently achieves the great reconstruction performance on all evaluated metrics. Notably, PnP-OneBit maintains strong performance even at severe compression ratios ($n/m=32$). 

Furthermore, visual comparisons of the reconstructed images are provided in \Cref{fig: QualitativeResults}. While methods seeking a single point estimate, such as Diff-OneBit, may oversmooth high-frequency information under the information bottleneck of one-bit quantization, PnP-OneBit mitigates this issue. By balancing the data likelihood and the diffusion prior via posterior sampling, it effectively retains high-frequency textures.
\begin{table}[t]
\centering
\caption{Quantitative results for one-bit recovery on the FFHQ dataset under different compression ratios \(n/m\) with additive Gaussian noise \(\sigma=0.5\). Best results are shown in \textbf{bold}.}
\label{tab: QuantitativeResultsFFHQ}
\setlength{\tabcolsep}{3.5pt}
\begin{tabular}{lcccccc}
\toprule
\multirow{2}{*}{Methods}  & \multicolumn{3}{c}{$n/m=16$}                & \multicolumn{3}{c}{$n/m=32$}                \\ 
\cmidrule(lr){2-4} \cmidrule(lr){5-7}
                         & PSNR$\uparrow$ & SSIM$\uparrow$ & LPIPS$\downarrow$ & PSNR$\uparrow$ & SSIM$\uparrow$ & LPIPS$\downarrow$  \\ 
\hline
DiffPIR                      & 11.71$\pm$1.26          & 0.30$\pm$.05          & 0.73$\pm$.06           & 11.48$\pm$0.63          &  0.24$\pm$.04         &  0.76$\pm$.05          \\
DPS                          & 14.88$\pm$1.21          & 0.52$\pm$.06          & 0.50$\pm$.07           & 13.81$\pm$0.78          &  0.50$\pm$.06         &  0.56$\pm$.05          \\
DAPS                         & 14.84$\pm$1.61          & 0.43$\pm$.06          & 0.59$\pm$.06           & 12.02$\pm$0.82          &  0.31$\pm$.06         &  0.61$\pm$.06          \\
QCS-SGM                      & 18.11$\pm$2.43          & 0.52$\pm$.10          & 0.52$\pm$.07           & 17.12$\pm$1.84          &  0.49$\pm$.09         &  0.59$\pm$.06          \\
SIM-DMIS                     & 18.73$\pm$1.94          & 0.52$\pm$.08          & 0.50$\pm$.07           & 17.36$\pm$1.34          &  0.48$\pm$.08         &  0.55$\pm$.08          \\
Diff-OneBit                  & 21.95$\pm$1.23          & 0.53$\pm$.08          & 0.47$\pm$.06           & 18.68$\pm$0.98          &  0.51$\pm$.06         &  0.53$\pm$.06          \\
\rowcolor{gray!20}PnP-OneBit & \textbf{22.96}$\pm$1.86 & \textbf{0.69}$\pm$.09 & \textbf{0.33}$\pm$.07  & \textbf{19.64}$\pm$1.23 & \textbf{0.58}$\pm$.07 & \textbf{0.49}$\pm$.07  \\
\bottomrule
\end{tabular}
\end{table}
\begin{figure}[t]
    \centering
    \includegraphics[width=1.0\linewidth]{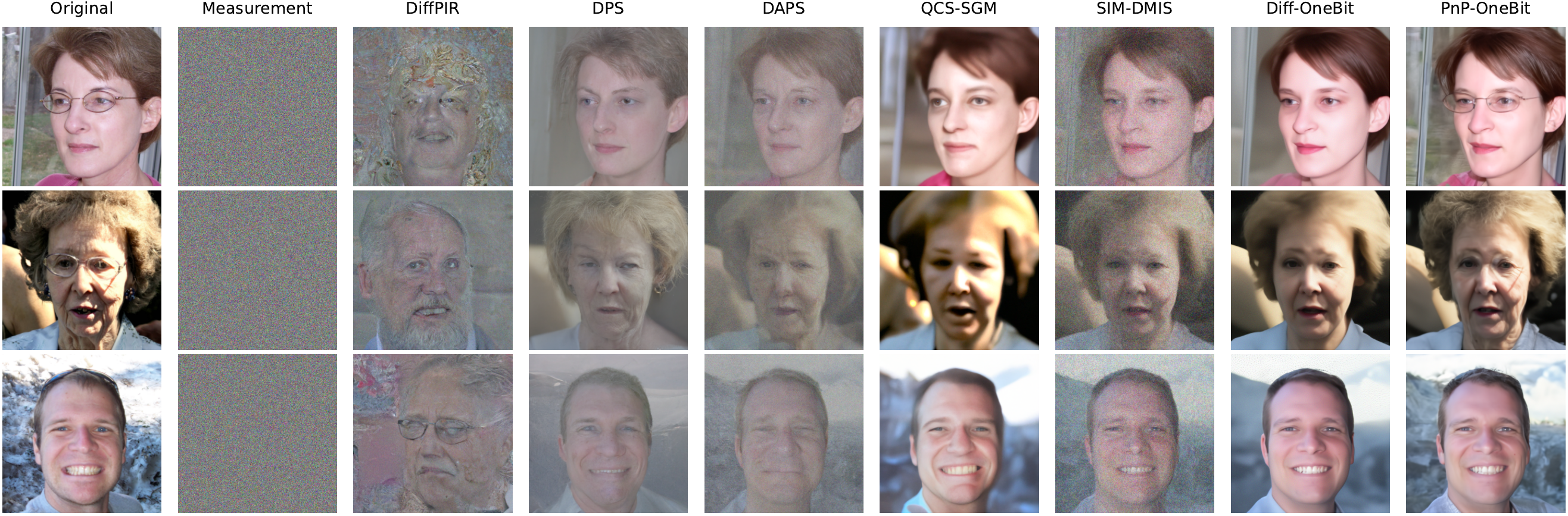}
    \caption{Qualitative results on FFHQ dataset with $\sigma=0.5$ and $n/m=16$.}
    \label{fig: QualitativeResults}
\end{figure}

\section{Conclusion}
We studied noisy one-bit compressed sensing for signals drawn from arbitrary data prior distributions on the sphere. We introduced an approximate covering number on the sphere and established a posterior sampling upper bound related to the logarithm of the approximate covering number and the recovery accuracy under noisy one-bit measurements, a robustness guarantee under learned prior mismatch, and an information theoretic lower bound for any reliable estimator. To complement these theoretical findings, we presented a one-bit plug-and-play reconstruction method and demonstrated its effectiveness on high dimensional image recovery tasks. Overall, our results show that reliable recovery under severe nonlinear quantization can be precisely characterized by the high probability geometric complexity of the signal prior. We focus on the noisy one-bit model with additive Gaussian noise before quantization. An open problem is whether similar robust posterior sampling bounds can be established when the model noise is adversarial.
\newpage
\bibliographystyle{plainnat}
\bibliography{refs}
\newpage
\appendix
\section{Upper bound proofs}
\label{upper_bound_proofs}

\paragraph{Proof roadmap.}
The proof consists of four steps. First, using the rotational invariance of Gaussian measurements, we show that the expected Hamming distance between one-bit observations is determined only by the geodesic distance between the underlying signals. Second, based on this relation, we establish a one-bit near--far test to quantify the separability between nearby and faraway signals in the one-bit observation space, and prove that its testing error decays exponentially with the number of measurements $m$. Third, we construct an approximate geodesic cover for the high probability mass of the signal distribution, and apply the above test together with a union bound over all covering centers, thereby obtaining the posterior sampling recovery guarantee in the matched-prior setting. Fourth, to handle model mismatch between the true distribution and the learned distribution, we use a geodesic Wasserstein truncation argument to decompose the two distributions into high probability regular components and small residual components, the regular components are close in $W_{\infty,\mathrm{geo}}$ and can be controlled by a discrete distribution whose support size is at most $e^k$. Finally, combining this with a R\'enyi stability lemma transfers the matched-prior recovery guarantee to the mismatched prior setting, while the residual components contribute only an additional failure probability of $2\delta$. %This proves \Cref{cor:main}.

\subsection{Proof of Lemma~\ref{hammingdistance}}
Define the pre-quantization continuous latent variables for the $i$-th measurement as $z_{1,i}=a_i^T x_1 + \xi_{i}$ and $z_{2,i}=a_i^T x_2$. According to Sheppard's formula~\citep{o2014analysis}, for jointly zero-mean Gaussian variables $z_{1,i}, z_{2,i}$, defining $\rho$ as the correlation coefficient between them, we have:
\begin{equation}
\Pr\left(\operatorname{sign}(z_{1,i}) \neq \operatorname{sign}(z_{2,i})\right) = \frac{1}{\pi} \arccos(\rho).
\end{equation}

Next, we consider calculating the correlation coefficient $\rho$ between $z_{1,i}, z_{2,i}$.

First, we calculate the variances of $z_{1,i}, z_{2,i}$. Since both $a_i^Tx_1$ and $\xi_{i}$ can be regarded as zero-mean Gaussian distributions and are mutually independent, by the superposition property of Gaussian distributions, we obtain:
\begin{align}
\label{eq:fangcha}
\operatorname{Var}(z_{1,i}) &= \operatorname{Var}(a_i^Tx_1)+\operatorname{Var}(\xi_{i}) , \\
&\ = \|x_1\|^2 +\sigma^2 , \\
&\ = 1+\sigma^2.
\end{align}
Similarly, since the second observation is noiseless, we obtain:
\begin{equation}
\operatorname{Var}(z_{2,i})=\|x_2\|_2^2=1.
\end{equation}

Next, we compute the covariance between the pre-quantization variables $z_{1,i}, z_{2,i}$. Since the noise $\xi_i$ is independent of $a_i$ and has mean zero:
\begin{align}
\label{eq:xiefangcha}
\operatorname{Cov}(z_{1,i}, z_{2,i})
&= \mathbb{E}[(a_i^Tx_1 + \xi_i)(a_i^Tx_2)], \\
&= \mathbb{E}[x_1^T a_i a_i^Tx_2] + \mathbb{E}[\xi_i a_i^T x_2], \\
&= \mathbb{E}[x_1^T a_i a_i^Tx_2], \\
&= x_1^Tx_2.
\end{align}

Utilizing the variance  Eq.~\eqref{eq:fangcha} and the covariance  Eq.~\eqref{eq:xiefangcha}, we get:
\begin{equation}
\rho= \frac{x_1^T x_2}{\sqrt{1+\sigma^2}}.
\end{equation}

Therefore, the normalized Hamming distance between the observation vectors $u,v$ is ultimately expressed as:
\begin{align}
\mathbb{E}[\mathrm{d}_{\mathrm{H}}(u,v)] &=  \frac{1}{m}\sum_{i=1}^m \Pr\left(\operatorname{sign}(z_{1,i}) \neq \operatorname{sign}(z_{2,i})\right) , \\
&=  \frac{1}{\pi} \arccos\left(\frac{\langle x_1, x_2 \rangle}{\sqrt{1+\sigma^2}}\right),
\\ &=\frac{1}{\pi} \arccos\left(\frac{\cos (\pi  \mathrm{d}_{\mathrm{S}}(x_1,x_2))}{\sqrt{1+\sigma^2}}\right).
\end{align}

\subsection{Proof of Lemma~\ref{lem:near_far_test}}
Fix an arbitrary \(x\in\mathbb S^{n-1}\). When the true signal is \(x\), write
\begin{equation}
y=\operatorname{sign}(Ax+\xi).
\end{equation}
For each \(i=1,\ldots,m\), define
\begin{equation}
Y_i
:=
\mathbf 1\left\{\operatorname{sign}(a_i^\top x+\xi_i)
\neq
\operatorname{sign}(a_i^\top x_0)\right\}.
\end{equation}
Then
\begin{equation}
\mathrm{d}_{\mathrm{H}}\left(y,\operatorname{sign}(Ax_0)\right)
=
\frac{1}{m}\sum_{i=1}^mY_i.
\end{equation}
Under the joint randomness of \((A,\xi)\) conditional on \(x\),
\((a_i,\xi_i)\) are independent and identically distributed. Therefore,
\(Y_1,\ldots,Y_m\) are independent and identically distributed Bernoulli random variables.

By Lemma~\ref{hammingdistance}, we have
\begin{equation}
\mathbb E_{a_i,\xi_i| x}[Y_i]
=
f_\sigma(\mathrm{d}_{\mathrm{S}}(x,x_0)).
\end{equation}
Note that \(f_\sigma(\cdot)\) is nondecreasing on \([0,1]\).

We first consider the near region. If \(\mathrm{d}_{\mathrm{S}}(x,x_0)\le \eta\), then
\begin{equation}
\mathbb E_{a_i,\xi_i| x}[Y_i]
\le
f_\sigma(\eta).
\end{equation}
Thus, by Hoeffding's inequality~\citep{bentkus2004hoeffding},
\begin{align}
\Pr_{A,\xi| x}
\left(\phi_{x_0,\eta}(y;A)=1\right)
&=
\Pr_{A,\xi| x}
\left(\mathrm{d}_{\mathrm{H}}\left(y,\operatorname{sign}(Ax_0)\right)\ge\tau_{\eta,c,\sigma}\right),  \\
&\le
\exp\left(-2m\left(\tau_{\eta,c,\sigma}-f_\sigma(\eta)\right)^2\right), \\
&=
\exp\left(-\frac{m}{2}\left(f_\sigma(c\eta)-f_\sigma(\eta)\right)^2\right), \\
&=
\exp\left(-\frac{m}{2}\Delta_\sigma(c,\eta)^2\right).
\end{align}
Taking the supremum over all \(x\) satisfying \(\mathrm{d}_{\mathrm{S}}(x,x_0)\le\eta\) yields the first claim.

We next consider the far region. If \(\mathrm{d}_{\mathrm{S}}(x,x_0)\ge c\eta\), then
\begin{equation}
\mathbb E_{a_i,\xi_i| x}[Y_i]
\ge
f_\sigma(c\eta).
\end{equation}
By Hoeffding's inequality~\citep{bentkus2004hoeffding},
\begin{align}
\Pr_{A,\xi| x}\left(\phi_{x_0,\eta}(y;A)=0\right)
&=
\Pr_{A,\xi| x}\left(\mathrm{d}_{\mathrm{H}}\left(y,\operatorname{sign}(Ax_0)\right)\le\tau_{\eta,c,\sigma}\right), \\
&\le
\exp\left(-2m\left(f_\sigma(c\eta)-\tau_{\eta,c,\sigma}\right)^2\right), \\
&=
\exp\left(-\frac{m}{2}\left(f_\sigma(c\eta)-f_\sigma(\eta)\right)^2\right), \\
&=
\exp\left(-\frac{m}{2}\Delta_\sigma(c,\eta)^2\right).
\end{align}
Taking the supremum over all \(x\) satisfying \(\mathrm{d}_{\mathrm{S}}(x,x_0)\ge c\eta\) yields the second claim.

\subsection{Proof of Lemma~\ref{lemma:orup}}
The proof proceeds in three steps. In the first step, we convert a possibly overlapping covering by spherical caps into a pairwise disjoint measurable partition. In the second step, for each partition block \(S_j\), we reduce the error event that the posterior sample falls into the far region to a binary posterior testing problem between the near region and the far region. In the third step, we sum over all \(j\), and use a counting argument to improve the final coefficient from the crude \(N+1\) to \(N\).

\label{thm:global_cover_bound_main}
Through \(x_1,\ldots,x_N\in\mathbb S^{n-1}\), we
define the far region
\begin{equation}
\mathcal F_j(c)
:=
\left\{
x\in\mathbb S^{n-1}: \mathrm{d}_{\mathrm{S}}(x,x_j)\ge c\eta
\right\}.
\end{equation}

Since these spherical caps $\mathcal C_j$ for $j=1,\ldots,N$ cover the whole mass of \(Q_g\), namely
\begin{equation}
Q_g\left[
\bigcup_{j=1}^N \mathcal C_j
\right]
=1 .
\end{equation}
Denote the covered region by
\begin{equation}
G
:=
\bigcup_{j=1}^N \mathcal C_j .
\end{equation}

Since the spherical caps \(\mathcal C_1,\ldots,\mathcal C_N\) may overlap, we define the following disjoint partition. Let
\begin{equation}
S_1:=\mathcal C_1,
\end{equation}
and for \(j=2,\ldots,N\), let
\begin{equation}
S_j
:=
\mathcal C_j
\setminus
\bigcup_{i=1}^{j-1}\mathcal C_i .
\end{equation}
By construction, \(S_1,\ldots,S_N\) are pairwise disjoint, and satisfy
\begin{equation}
S_j\subseteq \mathcal C_j,
\qquad
G=\bigsqcup_{j=1}^N S_j .
\end{equation}

Fix an arbitrary \(j\in\{1,\ldots,N\}\). If \(x^*\in S_j\), then since \(S_j\subseteq \mathcal C_j\), we have
\begin{equation}
\mathrm{d}_{\mathrm{S}}(x^*,x_j)\le \eta .
\end{equation}
On the other hand, if the recovery error satisfies
$\mathrm{d}_{\mathrm{S}}(x^*,\widehat x)\ge(c+1)\eta$,
then by the triangle inequality we obtain
\begin{align}
\mathrm{d}_{\mathrm{S}}(x_j,\widehat x)\ge
\mathrm{d}_{\mathrm{S}}(x^*,\widehat x)-\mathrm{d}_{\mathrm{S}}(x^*,x_j) \ge
(c+1)\eta-\eta =c\eta .
\end{align}
Therefore, we can conclude that when $\mathrm{d}_{\mathrm{S}}(x^*,\widehat{x})\ge (c+1)\eta$, we have $\widehat x\in \mathcal F_j(c)$.
Thus we obtain the event inclusion
\begin{equation}
\left\{
x^*\in S_j,\
\mathrm{d}_{\mathrm{S}}(x^*,\widehat x)\ge(c+1)\eta
\right\}
\subseteq
\left\{
x^*\in S_j,\
\widehat x\in \mathcal F_j(c)
\right\}.
\end{equation}

Since \(G=\bigsqcup_{j=1}^N S_j\), we have
\begin{equation}
\left\{
\mathrm{d}_{\mathrm{S}}(x^*,\widehat x)\ge(c+1)\eta
\right\}
\subseteq
\{x^*\notin G\}
\cup
\bigcup_{j=1}^N
\left\{
x^*\in S_j,\
\widehat x\in \mathcal F_j(c)
\right\}.
\end{equation}
Using the covering assumption \(Q_g(G)=1\), we obtain
\begin{align}
\label{eq.66}
\Pr
\left(
\mathrm{d}_{\mathrm{S}}(x^*,\widehat x)\ge(c+1)\eta
\right)
&\le
\Pr\left(x^*\notin G\right)
+
\sum_{j=1}^N
\Pr
\left(
x^*\in S_j,\
\widehat x\in \mathcal F_j(c)
\right), \\
&\le
\sum_{j=1}^N
\Pr
\left(
x^*\in S_j,\
\widehat x\in \mathcal F_j(c)
\right).
\end{align}
It remains to bound each term
$\Pr
\left(
x^*\in S_j,\
\widehat x\in \mathcal F_j(c)
\right)$.

Fix \(j\). If $Q_g(S_j)=0$
or
$Q(\mathcal F_j(c))=0$,
then
\begin{equation}
\Pr
\left(
x^*\in S_j,\
\widehat x\in \mathcal F_j(c)
\right)
=0,
\end{equation}
the claim is immediate. Hence in what follows we only consider the case $Q_g(S_j)>0,Q(\mathcal F_j(c))>0$.

Define
\begin{equation}
\pi_j^{\mathrm{in}}
:=
Q(S_j),
\qquad
\pi_j^{\mathrm{far}}
:=
Q(\mathcal F_j(c)).
\end{equation}
Since \(S_j\subseteq\mathcal C_j\) and \(c>1\), every \(x\in S_j\) satisfies $\mathrm{d}_{\mathrm{S}}(x,x_j)\le \eta<c\eta$.
Therefore
\begin{equation}
S_j\cap \mathcal F_j(c)=\varnothing .
\end{equation}
Define the middle region
\begin{equation}
M_j(c)
:=
\mathbb S^{n-1}
\setminus
\left(
S_j\cup \mathcal F_j(c)
\right),
\end{equation}
and let
\begin{equation}
\pi_j^{\mathrm{mid}}
:=
Q(M_j(c)).
\end{equation}
Then
\begin{equation}
\pi_j^{\mathrm{in}}
+
\pi_j^{\mathrm{far}}
+
\pi_j^{\mathrm{mid}}
=
1 .
\end{equation}

Let \(Q_j^{\mathrm{in}}\), \(Q_{j,c}^{\mathrm{far}}\), and \(Q_j^{\mathrm{mid}}\) denote the normalized restrictions of \(Q\) to \(S_j\), \(\mathcal F_j(c)\), and \(M_j(c)\), respectively. That is, for any Borel set \(B\subseteq\mathbb S^{n-1}\), define
\begin{equation}
Q_j^{\mathrm{in}}(B)
:=
\frac{Q(B\cap S_j)}{\pi_j^{\mathrm{in}}},
\end{equation}
and
\begin{equation}
Q_{j,c}^{\mathrm{far}}(B)
:=
\frac{Q(B\cap \mathcal F_j(c))}{\pi_j^{\mathrm{far}}}.
\end{equation}
When \(\pi_j^{\mathrm{mid}}>0\), further define
\begin{equation}
Q_j^{\mathrm{mid}}(B)
:=
\frac{Q(B\cap M_j(c))}{\pi_j^{\mathrm{mid}}}.
\end{equation}

For fixed \(A\), write
\begin{equation}
p_{A,x}(y)
:=
\Pr\left(y| A,x\right).
\end{equation}
Accordingly, define the marginal likelihoods induced by the near region, the far region, and the middle region:
\begin{equation}
h_j^{\mathrm{in}}(y;A)
:=
\int p_{A,x}(y)\,dQ_j^{\mathrm{in}}(x),
\end{equation}
\begin{equation}
h_j^{\mathrm{far}}(y;A)
:=
\int p_{A,x}(y)\,dQ_{j,c}^{\mathrm{far}}(x),
\end{equation}
and, when \(\pi_j^{\mathrm{mid}}>0\),
\begin{equation}
h_j^{\mathrm{mid}}(y;A)
:=
\int p_{A,x}(y)\,dQ_j^{\mathrm{mid}}(x).
\end{equation}
Let \(Q_{g,j}^{\mathrm{in}}\) denote the normalized restriction of \(Q_g\) to \(S_j\), and let \(h_{g,j}^{\mathrm{in}}(y;A):=\int p_{A,x}(y)\,dQ_{g,j}^{\mathrm{in}}(x)\).

Given \(y\) and \(A\), the posterior probability assigned to the far region \(\mathcal F_j(c)\) is
\begin{equation}
Q(\mathcal F_j(c)| y,A)
=
\frac{
\pi_j^{\mathrm{far}}h_j^{\mathrm{far}}(y;A)
}{
\pi_j^{\mathrm{in}}h_j^{\mathrm{in}}(y;A)
+
\pi_j^{\mathrm{far}}h_j^{\mathrm{far}}(y;A)
+
\pi_j^{\mathrm{mid}}h_j^{\mathrm{mid}}(y;A)
}.
\end{equation}
Since the middle term satisfies
\begin{equation}
\pi_j^{\mathrm{mid}}h_j^{\mathrm{mid}}(y;A)\ge 0,
\end{equation}
Dropping this nonnegative term from the denominator can only increase the ratio. Therefore
\begin{equation}
Q(\mathcal F_j(c)| y,A)
\le
\frac{
\pi_j^{\mathrm{far}}h_j^{\mathrm{far}}(y;A)
}{
\pi_j^{\mathrm{in}}h_j^{\mathrm{in}}(y;A)
+
\pi_j^{\mathrm{far}}h_j^{\mathrm{far}}(y;A)
}.
\end{equation}

Define the far-region posterior probability under the reduced two-group model by
\begin{equation}
\Gamma_j^{\mathrm{far}}(y,A)
:=
\frac{
\pi_j^{\mathrm{far}}h_j^{\mathrm{far}}(y;A)
}{
\pi_j^{\mathrm{in}}h_j^{\mathrm{in}}(y;A)
+
\pi_j^{\mathrm{far}}h_j^{\mathrm{far}}(y;A)
}.
\end{equation}
Thus
\begin{align}
\Pr
\left(
x^*\in S_j,\
\widehat x\in \mathcal F_j(c)
\right)
&=
\mathbb E_A
\sum_{y\in\{-1,+1\}^m}
Q_g(S_j)
h_{g,j}^{\mathrm{in}}(y;A)
Q(\mathcal F_j(c)| y,A), \\
&\le
Q_g(S_j)
\mathbb E_A
\sum_{y\in\{-1,+1\}^m}
h_{g,j}^{\mathrm{in}}(y;A)
\Gamma_j^{\mathrm{far}}(y,A).
\end{align}

Now let $\phi:\{-1,+1\}^m\times\mathbb R^{m\times n}\to[0,1]$ be any measurable test function. Since \(Q=(1-\alpha)Q_g+\alpha Q_b\), we have
\begin{equation}
\pi_j^{\mathrm{in}}h_j^{\mathrm{in}}(y;A)
=
\int_{S_j}p_{A,x}(y)\,dQ(x)
\ge
(1-\alpha)
\int_{S_j}p_{A,x}(y)\,dQ_g(x)
=
(1-\alpha)Q_g(S_j)h_{g,j}^{\mathrm{in}}(y;A).   
\end{equation}

For any fixed \((y,A)\), we have
\begin{equation}
Q_g(S_j)
h_{g,j}^{\mathrm{in}}(y;A)
\Gamma_j^{\mathrm{far}}(y,A)
\le
\frac{1}{1-\alpha}
\left(
\pi_j^{\mathrm{in}}h_j^{\mathrm{in}}(y;A)\phi(y;A)
+
\pi_j^{\mathrm{far}}h_j^{\mathrm{far}}(y;A)
\left(
1-\phi(y;A)
\right)
\right).
\end{equation}
This is because on the region selected by \(\phi\), we use the trivial upper bound
\begin{equation}
\Gamma_j^{\mathrm{far}}(y,A)\le 1;
\end{equation}
whereas on the complementary region, we use
\begin{equation}
Q_g(S_j)
h_{g,j}^{\mathrm{in}}(y;A)
\Gamma_j^{\mathrm{far}}(y,A)
\le
\frac{1}{1-\alpha}
\pi_j^{\mathrm{far}}
h_j^{\mathrm{far}}(y;A).
\end{equation}

Therefore
\begin{equation}
\Pr
\left(
x^*\in S_j,\
\widehat x\in \mathcal F_j(c)
\right)
\le
\frac{1}{1-\alpha}
\left(
\pi_j^{\mathrm{in}}\alpha_j(\phi)
+
\pi_j^{\mathrm{far}}\beta_j(\phi)
\right),
\end{equation}
where
\begin{equation}
\alpha_j(\phi)
:=
\mathbb E_A
\sum_{y\in\{-1,+1\}^m}
h_j^{\mathrm{in}}(y;A)\phi(y;A),
\end{equation}
and
\begin{equation}
\beta_j(\phi)
:=
\mathbb E_A
\sum_{y\in\{-1,+1\}^m}
h_j^{\mathrm{far}}(y;A)
\left(
1-\phi(y;A)
\right).
\end{equation}

Next, choose the near-far Hamming test centered at \(x_j\). Define
\begin{equation}
H_j(y;A)
:=
\mathrm{d}_{\mathrm{H}}
\left(
y,\operatorname{sign}(Ax_j)
\right),
\end{equation}
and let
\begin{equation}
\phi_j(y;A)
:=
\mathbf 1
\left\{
H_j(y;A)
\ge
\frac{
f_\sigma(\eta)+f_\sigma(c\eta)
}{2}
\right\}.
\end{equation}

Since \(Q_j^{\mathrm{in}}\) is supported on \(S_j\subseteq\mathcal C_j\), for any
\(x\in\operatorname{supp}(Q_j^{\mathrm{in}})\), we have
\begin{equation}
\mathrm{d}_{\mathrm{S}}(x,x_j)\le \eta.
\end{equation}
By Lemma~\ref{lem:near_far_test}, we obtain
\begin{equation}
\alpha_j(\phi_j)
\le
\exp
\left(
-\frac{m}{2}\Delta_\sigma(c,\eta)^2
\right).
\end{equation}
On the other hand, \(Q_{j,c}^{\mathrm{far}}\) is supported on \(\mathcal F_j(c)\), so for any
\(x\in\operatorname{supp}(Q_{j,c}^{\mathrm{far}})\), we have
\begin{equation}
\mathrm{d}_{\mathrm{S}}(x,x_j)\ge c\eta.
\end{equation}
Again by Lemma~\ref{lem:near_far_test}, we obtain
\begin{equation}
\beta_j(\phi_j)
\le
\exp
\left(
-\frac{m}{2}\Delta_\sigma(c,\eta)^2
\right).
\end{equation}
Therefore
\begin{align}
\label{eq.101}
\Pr
\left(
x^*\in S_j,\
\widehat x\in \mathcal F_j(c)
\right)
&\le
\frac{1}{1-\alpha}
\left(
\pi_j^{\mathrm{in}}
+
\pi_j^{\mathrm{far}}
\right)
\exp
\left(
-\frac{m}{2}\Delta_\sigma(c,\eta)^2
\right) \\
&=
\frac{1}{1-\alpha}
\left(
Q(S_j)
+
Q(\mathcal F_j(c))
\right)
\exp
\left(
-\frac{m}{2}\Delta_\sigma(c,\eta)^2
\right).
\end{align}

Substituting the result of~\eqref{eq.101} into the event decomposition in~\eqref{eq.66} gives
\begin{equation}
\Pr_{x^*,A,\xi,\widehat{x}}
\left(
\mathrm{d}_{\mathrm{S}}(x^*,\widehat x)\ge(c+1)\eta
\right)
\le
\frac{1}{1-\alpha}
\sum_{j=1}^N
\left(
Q(S_j)+Q(\mathcal F_j(c))
\right)
\exp
\left(
-\frac{m}{2}\Delta_\sigma(c,\eta)^2
\right).
\end{equation}
Since \(S_1,\ldots,S_N\) form a disjoint partition of \(G\), we have
\begin{equation}
\sum_{j=1}^N Q(S_j)=Q(G).
\end{equation}
Thus
\begin{equation}
\Pr_{x^*,A,\xi,\widehat{x}}
\left(
\mathrm{d}_{\mathrm{S}}(x^*,\widehat x)\ge(c+1)\eta
\right)
\le
\frac{1}{1-\alpha}
\left(
Q(G)+\sum_{j=1}^N Q(\mathcal F_j(c))
\right)
\exp
\left(
-\frac{m}{2}\Delta_\sigma(c,\eta)^2
\right).
\end{equation}

We next prove that $Q(G)+\sum_{j=1}^N Q(\mathcal F_j(c))
\le N$.
Observe that
\begin{align}
Q(G)+\sum_{j=1}^N Q(\mathcal F_j(c))
&=
\int_{\mathbb S^{n-1}}
\left(
\mathbf 1_G(x)
+
\sum_{j=1}^N
\mathbf 1_{\mathcal F_j(c)}(x)
\right)
dQ(x).
\end{align}
For any \(x\in G\), there exists at least one \(k\in\{1,\ldots,N\}\) such that \(x\in\mathcal C_k\). Therefore
\begin{equation}
\mathrm{d}_{\mathrm{S}}(x,x_k)\le \eta<c\eta,
\end{equation}
where the last step follows from \(c>1\). Hence
\begin{equation}
x\notin \mathcal F_k(c).
\end{equation}
Therefore, when \(x\in G\), the point \(x\) belongs to at most \(N-1\) far regions
\(\mathcal F_1(c),\ldots,\mathcal F_N(c)\). Thus
\begin{equation}
\mathbf 1_G(x)
+
\sum_{j=1}^N
\mathbf 1_{\mathcal F_j(c)}(x)
\le
1+(N-1)
=
N.
\end{equation}
When \(x\notin G\), it is clear that
\begin{equation}
\mathbf 1_G(x)
+
\sum_{j=1}^N
\mathbf 1_{\mathcal F_j(c)}(x)
=
\sum_{j=1}^N
\mathbf 1_{\mathcal F_j(c)}(x)
\le N.
\end{equation}
Therefore, for any \(x\in\mathbb S^{n-1}\), we have
\begin{equation}
\mathbf 1_G(x)
+
\sum_{j=1}^N
\mathbf 1_{\mathcal F_j(c)}(x)
\le N.
\end{equation}
Integrating the above inequality with respect to \(Q\), we obtain
\begin{equation}
Q(G)+\sum_{j=1}^N Q(\mathcal F_j(c))
\le N.
\end{equation}
Therefore
\begin{equation}
\Pr_{x^*,A,\xi,\widehat{x}}
\left(
\mathrm{d}_{\mathrm{S}}(x^*,\widehat x)\ge(c+1)\eta
\right)
\le
\frac{N}{1-\alpha}
\exp
\left(
-\frac{m}{2}\Delta_\sigma(c,\eta)^2
\right).
\end{equation}
The proof is complete.

\begin{lemma}(Adopted from~\citep[Lemma A.1]{jalal2021instance}) 
\label{lem:cov_lemma}
Let $R, P$ be arbitrary distributions on $\mathbb{S}^{n-1}$. Let $p \ge 1$ and $\eta, \rho, \delta > 0$ be parameters.
If $\mathcal{W}_{p,geo}(R, P) \le \rho$ and $\min\{\log \mathrm{Cov}_{\eta, \delta}(P), \log \mathrm{Cov}_{\eta, \delta}(R)\} \le k$, then there exist distributions $R', R'', P', P''$ and a finite discrete distribution $Q$ with a support size $|supp(Q)| \le e^k$, satisfying the following conditions:

\begin{enumerate}
\item $\min \{\mathcal{W}_{\infty,geo}(P', Q), \mathcal{W}_{\infty,geo}(R', Q)\} \le \eta$.
\item $\mathcal{W}_{\infty,geo}(R', P') \le \frac{\rho}{\delta^{1/p}}$.
\item $P = (1-2\delta)P' + (2\delta)P''$ and $R = (1-2\delta)R' + (2\delta)R''$.
\end{enumerate}
\end{lemma}

%\renyi*

\subsection{Proof of Lemma~\ref{lem:renyi_onebit_mismatch_explicit}}
The central idea of the proof is to relate the true signal $x^*$ to the auxiliary signal $z^*$ through an optimal coupling associated with
$\mathcal W_{\infty,\mathrm{geo}}(R_0,S_0)$.
Since $x^*$ and $z^*$ differ by at most $\epsilon$ in spherical geodesic distance, if the posterior-generated signal $\widehat x$ is at distance at least $d+\epsilon$ from the true signal $x^*$, then it is at distance at least $d$ from the auxiliary signal $z^*$. Therefore, the error event under the true model can first be reduced to a tail event centered at the auxiliary signal $z^*$.
We then use a R\'enyi/H\"older change of measure to convert the posterior sampling error under the true observation $y$ into the corresponding posterior sampling error under the auxiliary observation $u$.

First, by the choice of the coupling $\Pi$, we have
\begin{equation}
    \mathrm{d}_{\mathrm{S}}(x^*,z^*)\le \epsilon
    \qquad \Pi\text{-almost surely}.
\end{equation}
Therefore, under this coupling, the spherical triangle inequality implies that if
\begin{equation}
    \mathrm{d}_{\mathrm{S}}(x^*,\widehat x)\ge d+\epsilon,
\end{equation}
then
\begin{equation}
    \mathrm{d}_{\mathrm{S}}(z^*,\widehat x)
    \ge
    \mathrm{d}_{\mathrm{S}}(x^*,\widehat x)-\mathrm{d}_{\mathrm{S}}(x^*,z^*)
    \ge d.
\end{equation}
Thus, we have the event inclusion
\begin{equation}
    \left\{
    \mathrm{d}_{\mathrm{S}}(x^*,\widehat x)\ge d+\epsilon
    \right\}
    \subseteq
    \left\{
    \mathrm{d}_{\mathrm{S}}(z^*,\widehat x)\ge d
    \right\}.
\end{equation}
Consequently,
\begin{equation}
\label{eq:event_reduction}
    \Pr\left(
    \mathrm{d}_{\mathrm{S}}(x^*,\widehat x)\ge d+\epsilon
    \right)
    \le
    \Pr\left(
    \mathrm{d}_{\mathrm{S}}(z^*,\widehat x)\ge d
    \right).
\end{equation}

Next, fix a realization $z\in\mathbb S^{n-1}$ of the auxiliary signal and define the tail set centered at $z$ by
\begin{equation}
    E_z
    :=
    \left\{
    w\in\mathbb S^{n-1}:
    \mathrm{d}_{\mathrm{S}}(z,w)\ge d
    \right\}.
\end{equation}
Let $K_A(s,\cdot)$ denote the posterior kernel induced by the prior $D$ and the one-bit likelihood, that is, for any measurable set
$B\subseteq\mathbb S^{n-1}$,
\begin{equation}
   K_A(s,B)
    :=
    \Pr_{w\sim D}\!\left(
        w\in B
        \,\middle|\,
        \operatorname{sign}(Aw+\xi)=s,\ A
    \right).
\end{equation}
More explicitly, if $\mu_{w,A}$ denotes the distribution of the one-bit observation
$\operatorname{sign}(Aw+\xi)$ conditional on $w$ and $A$, then
\begin{equation}
    K_A(s,B)
    =
    \frac{
    \int_B \mu_{w,A}(s)\,dD(w)
    }{
    \int_{\mathbb S^{n-1}}\mu_{v,A}(s)\,dD(v)
    }.
\end{equation}
It follows that $\widehat{x},\widehat{z}$ satisfy
\begin{equation}
    \widehat x\sim K_A(y,\cdot),
    \qquad
    \widehat z\sim K_A(u,\cdot).
\end{equation}
After conditioning on the random variables $x^*=x$ and $z^*=z$, the set $E_z$ is fixed, while the remaining randomness comes from the measurement matrix $A$, the noise $\xi$, and the posterior sample $\widehat x$ induced by the observation $y$.
Let $\mathsf M_x^{(m)}$ denote the distribution of the joint random variable
$(A,y)=\left(A,\operatorname{sign}(Ax+\xi)\right)$ when the true signal is fixed at $x$. Then
\begin{equation}
\label{eq:conditional_integral_representation}
    \Pr\left(
    \mathrm{d}_{\mathrm{S}}(z,\widehat x)\ge d
    | x^*=x,z^*=z
    \right)
    =
    \int K_A(y,E_z)\,d\mathsf M_x^{(m)}(A,y).
\end{equation}
This is because, conditional on $(A,y)$, the posterior-generated signal satisfies $\widehat x\sim K_A(y,\cdot)$, and hence
\begin{equation}
    \Pr\left(
    \widehat x\in E_z
    | A,y
    \right)
    =
    K_A(y,E_z).
\end{equation}

Similarly, let $\mathsf M_z^{(m)}$ denote the distribution of the joint random variable
$(A,u)=\left(A,\operatorname{sign}(Az+\xi')\right)$
when the auxiliary signal is fixed at $z$.

We now perform a change of measure for fixed $x,z\in\mathbb S^{n-1}$. Since the Gaussian noise assigns positive probability to every one-bit pattern, we have
\begin{equation}
    \mathsf M_x^{(m)}
    \ll
    \mathsf M_z^{(m)}.
\end{equation}
Denote the corresponding likelihood ratio by
\begin{equation}
    L_{x,z}(A,s)
    :=
    \frac{d\mathsf M_x^{(m)}}{d\mathsf M_z^{(m)}}(A,s),
    \qquad
    s\in\{-1,+1\}^m.
\end{equation}
Thus, by Eq.~\eqref{eq:conditional_integral_representation}, we obtain
\begin{align}
    \int K_A(y,E_z)\,d\mathsf M_x^{(m)}(A,y)
    &=
    \int K_A(s,E_z)
    L_{x,z}(A,s)
    \,d\mathsf M_z^{(m)}(A,s).
\end{align}

Let
\begin{equation}
    q:=1+\lambda,
    \qquad
    \theta:=\frac{q-1}{q}
    =
    \frac{\lambda}{1+\lambda}.
\end{equation}
Applying H\"older's inequality to the preceding display gives
\begin{align}
    \int K_A(s,E_z)\,d\mathsf M_x^{(m)}(A,s)
    &\le
    \left(
    \int L_{x,z}(A,s)^q
    \,d\mathsf M_z^{(m)}(A,s)
    \right)^{1/q}
    \left(
    \int K_A(s,E_z)^{q/(q-1)}
    \,d\mathsf M_z^{(m)}(A,s)
    \right)^{(q-1)/q}.
\end{align}
Since$0\le K_A(s,E_z)\le 1$ and $q/(q-1)>1$, we have
\begin{equation}
    K_A(s,E_z)^{q/(q-1)}
    \le
    K_A(s,E_z).
\end{equation}
Therefore,
\begin{align}
\label{eq:holder_joint_measure}
    \int K_A(s,E_z)\,d\mathsf M_x^{(m)}(A,s)
    &\le
    \left(
    \int L_{x,z}(A,s)^q
    \,d\mathsf M_z^{(m)}(A,s)
    \right)^{1/q}
    p_z^\theta,
\end{align}
where
\begin{equation}
    p_z
    :=
    \int K_A(s,E_z)\,d\mathsf M_z^{(m)}(A,s).
\end{equation}
Here, $p_z$ denotes the probability that the posterior-generated signal $\widehat z$ falls into the tail set $E_z$ when the auxiliary signal is fixed at $z$ and the observation is
$u=\operatorname{sign}(Az+\xi')$.

We next control the R\'enyi-type moment in Eq.~\eqref{eq:holder_joint_measure}. To this end, introduce the unquantized joint Gaussian observation laws.
Let $\widetilde{\mathsf M}_x^{(m)}$ denote the distribution of the joint random variable
$(A,V_x)=(A,Ax+\xi)$,
and let $\widetilde{\mathsf M}_z^{(m)}$ denote the distribution of the joint random variable
$(A,V_z)=(A,Az+\xi')$.
Since the one-bit observation is obtained through the mapping
\begin{equation}
    (A,v)\mapsto (A,\operatorname{sign}(v)),
\end{equation}
By the data processing inequality for classical R\'enyi divergence of order q>1, applied to the measurable map $(A,v)\to(A,sign(v))$,we have
\begin{equation}
    D_q\left(
    \mathsf M_x^{(m)}
    \Vert
    \mathsf M_z^{(m)}
    \right)
    \le
    D_q\left(
    \widetilde{\mathsf M}_x^{(m)}
    \Vert
    \widetilde{\mathsf M}_z^{(m)}
    \right).
\end{equation}
Equivalently,
\begin{align}
\label{eq:data_processing_lq}
    \int L_{x,z}(A,s)^q
    \,d\mathsf M_z^{(m)}(A,s)
    &\le
    \int
    \left(
    \frac{
    d\widetilde{\mathsf M}_x^{(m)}
    }{
    d\widetilde{\mathsf M}_z^{(m)}
    }(A,v)
    \right)^q
    d\widetilde{\mathsf M}_z^{(m)}(A,v).
\end{align}

We now compute the right-hand side. Since $A$ has the same marginal distribution under both laws, and conditional on $A$,
\begin{equation}
    V_x| A
    \sim
    \mathcal N(Ax,\sigma^2I_m),
    \qquad
    V_z| A
    \sim
    \mathcal N(Az,\sigma^2I_m),
\end{equation}
the R\'enyi moment formula for Gaussian distributions with the same covariance gives
\begin{align}
&\int
    \left(
    \frac{
    d\widetilde{\mathsf M}_x^{(m)}
    }{
    d\widetilde{\mathsf M}_z^{(m)}
    }(A,v)
    \right)^q
    d\widetilde{\mathsf M}_z^{(m)}(A,v)
    \nonumber\\
&\qquad
=
    \mathbb E_A
    \exp\left(
    \frac{q(q-1)}{2\sigma^2}
    \|A(x-z)\|_2^2
    \right).
\end{align}
Let
\begin{equation}
    h:=x-z.
\end{equation}
If $h\neq0$, since the rows of $A$ are independent standard Gaussian vectors, we have
\begin{equation}
    \frac{\|Ah\|_2^2}{\|h\|_2^2}
    \sim
    \chi_m^2.
\end{equation}
Hence, as long as$\frac{q(q-1)\|h\|_2^2}{\sigma^2}<1$,
we have
\begin{align}
    \mathbb E_A
    \exp\left(
    \frac{q(q-1)}{2\sigma^2}
    \|Ah\|_2^2
    \right)
    &=
    \left(
    1-
    \frac{q(q-1)\|h\|_2^2}{\sigma^2}
    \right)^{-m/2}.
\end{align}
If $h=0$, then the above exponential moment equals $1$, and the formula can be understood by continuous extension.

We now control the Euclidean distance by the spherical mismatch radius. Since $x,z\in\mathbb S^{n-1}$, we have
\begin{equation}
    \|x-z\|_2
    =
    2\sin\left(
    \frac{\pi \mathrm{d}_{\mathrm{S}}(x,z)}{2}
    \right).
\end{equation}
Moreover, under the coupling $\Pi$, we have $\mathrm{d}_{\mathrm{S}}(x,z)\le\epsilon$, and since $ 0<\epsilon\le1 $,
\begin{equation}
    \|x-z\|_2
    \le
    2\sin\left(
    \frac{\pi\epsilon}{2}
    \right)
    =
    r_\epsilon.
\end{equation}
Combining this with $q=1+\lambda$ and the condition$\lambda(1+\lambda)r_\epsilon^2<\sigma^2$,we obtain
\begin{equation}
    \frac{q(q-1)\|x-z\|_2^2}{\sigma^2}
    \le
    \frac{\lambda(1+\lambda)r_\epsilon^2}{\sigma^2}
    <
    1.
\end{equation}
Therefore,
\begin{align}
    \left(
    \int L_{x,z}(A,s)^q
    \,d\mathsf M_z^{(m)}(A,s)
    \right)^{1/q}
    &\le
    \left(
    1-
    \frac{\lambda(1+\lambda)r_\epsilon^2}{\sigma^2}
    \right)^{-\frac{m}{2(1+\lambda)}}.
\end{align}
Substituting this into Eq.~\eqref{eq:holder_joint_measure}, we obtain, for any fixed signal pair $x,z$ satisfying $\mathrm{d}_{\mathrm{S}}(x,z)\le\epsilon$,
\begin{align}
\label{eq:fixed_pair_bound}
    \int K_A(s,E_z)\,d\mathsf M_x^{(m)}(A,s)
    &\le
    \left(
    1-\frac{\lambda(1+\lambda)r_\epsilon^2}{\sigma^2}
    \right)^{-\frac{m}{2(1+\lambda)}}
    p_z^{\frac{\lambda}{1+\lambda}}.
\end{align}

Finally, we take expectation with respect to the coupling $(x^*,z^*)\sim\Pi$. By
Eq.~\eqref{eq:event_reduction} and Eq.~\eqref{eq:fixed_pair_bound}, we get
\begin{align}
    \Pr\left(
    \mathrm{d}_{\mathrm{S}}(x^*,\widehat x)\ge d+\epsilon
    \right)
    &\le
    \left(
    1-\frac{\lambda(1+\lambda)r_\epsilon^2}{\sigma^2}
    \right)^{-\frac{m}{2(1+\lambda)}}
    \mathbb E_{z^*}
    \left[
    p_{z^*}^{\frac{\lambda}{1+\lambda}}
    \right].
\end{align}
Since
\begin{equation}
    0<\frac{\lambda}{1+\lambda}<1,
\end{equation}
the function $t\mapsto t^{\lambda/(1+\lambda)}$ is concave. Therefore, by Jensen's inequality,
\begin{equation}
    \mathbb E_{z^*}
    \left[
    p_{z^*}^{\frac{\lambda}{1+\lambda}}
    \right]
    \le
    \left(
    \mathbb E_{z^*}[p_{z^*}]
    \right)^{\frac{\lambda}{1+\lambda}}.
\end{equation}
By the definition of $p_z$,
\begin{equation}
    \mathbb E_{z^*}[p_{z^*}]
    =
    \Pr\left(
    \mathrm{d}_{\mathrm{S}}(z^*,\widehat z)\ge d
    \right),
\end{equation}
where
\begin{equation}
    u=\operatorname{sign}(Az^*+\xi'),
    \qquad
    \widehat z\sim K_A(u,\cdot).
\end{equation}
Hence,
\begin{equation}
    \Pr\left(
    \mathrm{d}_{\mathrm{S}}(x^*,\widehat x)\ge d+\epsilon
    \right)
    \le
    \left(
    1-\frac{\lambda(1+\lambda)r_\epsilon^2}{\sigma^2}
    \right)^{-\frac{m}{2(1+\lambda)}}
    \Pr\left(
    \mathrm{d}_{\mathrm{S}}(z^*,\widehat z)\ge d
    \right)^{\frac{\lambda}{1+\lambda}}.
\end{equation}
This completes the proof.

\subsection{Proof of Theorem~\ref{cor:main}}

Based on the above lemmas, we derive the following lemma, which is closely related to Theorem~\ref{cor:main}.
\begin{lemma}
\label{lem:main}
Let \(R\) and \(P\) be two probability distributions on the unit sphere \(\mathbb S^{n-1}\). Let \(p\ge 1\), \(0<\eta<1\), \(0<\delta<1/4\), and assume that
$\mathcal W_{p,\mathrm{geo}}(R,P)\le \rho $.

Define the effective mismatch level
$\varepsilon_\delta:=\rho/\delta^{1/p}$,
and define the parameter $
\bar\eta
:=
\eta+\varepsilon_\delta $.
We take $c\in\left(1,\frac{1-\varepsilon_\delta}{\bar\eta}-1\right)$, assume further that \(2\eta+3\varepsilon_\delta<1\), which guarantees that
$\left(1,\frac{1-\varepsilon_\delta}{\bar\eta}-1\right)$
is a nonempty interval. Given \((y,A)\),
$\widehat x
\sim
P(\cdot| y,A)$
is the posterior sampling estimator based on \(P\).
Furthermore, define
\(r_{\varepsilon_\delta}
:=
2\sin\left(
\frac{\pi\varepsilon_\delta}{2}
\right)\), for any $
0<\lambda\le
\min\left\{
1,\,
\frac{\sigma^2}{3r_{\varepsilon_\delta}^2}
\right\}$
(if $\rho=0$ ,we take the $\lambda=1$), and any \(\beta\in(0,1)\), with probability at least \(1-\beta\) over the draw of \(A\), we have
\begin{align}
\label{eq:main_wp_truncated_bound_cn}
&
\Pr_{x^*\sim R,\xi,\widehat{x}| A}\left(
\mathrm{d}_{\mathrm{S}}(x^*,\widehat x)
\ge
(c+1)\bar\eta+\varepsilon_\delta
\right)
\nonumber\\
&\qquad\le
2\delta
+
\frac{1}{\beta}
\left(
1-
\frac{\lambda(1+\lambda)r_{\varepsilon_\delta}^2}{\sigma^2}
\right)^{-\frac{m}{2(1+\lambda)}}
\exp\!\left(
\frac{\lambda}{1+\lambda}
\left(
\log \mathrm{Cov}_{\eta,\delta}(R)+\log{\frac{1}{1-2\delta}}-\frac m2\Delta_\sigma(c,\bar\eta)^2
\right)
\right).
\end{align}
In particular, when \(R=P\), we have \(\rho=0\). Taking \(\lambda=1\), we obtain
\begin{equation}
\Pr_{x^*\sim R,\xi,\widehat{x}| A}\left(
\mathrm{d}_{\mathrm{S}}(x^*,\widehat x)
\ge
(c+1)\eta
\right)\le2\delta+\frac{1}{\beta}
\exp\!\left(
\frac{1}{2}
\left(
\log \mathrm{Cov}_{\eta,\delta}(R)+\log{\frac{1}{1-2\delta}}-\frac m2\Delta_\sigma(c,\eta)^2
\right)
\right).
\end{equation}
\end{lemma}

\begin{proof}

By assumption, \(\mathcal W_{p,\mathrm{geo}}(R,P)\le \rho\),
and we also have
\begin{equation}
\min\{
\log \mathrm{Cov}_{\eta,\delta}(P),
\log \mathrm{Cov}_{\eta,\delta}(R)
\}
\le \log \mathrm{Cov}_{\eta,\delta}(R).
\end{equation}
By \(\Cref{lem:cov_lemma}\), there exist probability distributions
\(R'\), \(R''\), \(P'\), \(P''\), and a finite discrete distribution \(Q\), such that
\begin{equation}
|\operatorname{supp}(Q)|\le \mathrm{Cov}_{\eta,\delta}(R),
\end{equation}
and
\begin{equation}
\label{eq:good_winfty_mismatch_cn}
\mathcal W_{\infty,\mathrm{geo}}(R',P')
\le
\frac{\rho}{\delta^{1/p}}
=
\varepsilon_\delta .
\end{equation}
Moreover, \(R\) and \(P\) can be decomposed as
\begin{equation}
\label{eq:mixture_decomp_cn}
R=(1-2\delta)R' + 2\delta R'',
\qquad
P=(1-2\delta)P' + 2\delta P'' .
\end{equation}
In addition, the lemma also guarantees that
\begin{equation}
\label{eq:min_q_close_cn}
\min\{
\mathcal W_{\infty,\mathrm{geo}}(P',Q),
\mathcal W_{\infty,\mathrm{geo}}(R',Q)
\}
\le \eta .
\end{equation}

This step converts the global \(W_p\) mismatch into a
\(W_\infty\) mismatch on the good parts. The residual components \(R''\) and \(P''\)
may be arbitrary distributions, but their weights in the original distributions are at most \(2\delta\).

We next prove that \(P'\) can be fully covered by at most \(\mathrm{Cov}_{\eta,\delta}(R)\) spherical caps of radius \(\bar\eta\).

By Eq.~\eqref{eq:min_q_close_cn}, there are two cases.

The first case is \(\mathcal W_{\infty,\mathrm{geo}}(P',Q)\le \eta\).

By the definition of the \(W_\infty\) distance, there exists a coupling
\(\gamma\in\Pi(P',Q)\) such that
\begin{equation}
    \mathrm{d}_{\mathrm{S}}(x,q)\le \eta
    \qquad \gamma\text{-a.s.}
\end{equation}
Since \(Q\) is a finite discrete distribution and the second marginal of \(\gamma\) is \(Q\),
we have
\begin{equation}
    \gamma\left(
    \left\{(x,q): q\in \operatorname{supp}(Q),
    \ \mathrm{d}_{\mathrm{S}}(x,q)\le \eta \right\}
    \right)=1 .
\end{equation}
Therefore,
\begin{equation}
    P'\left(
    \bigcup_{q\in\operatorname{supp}(Q)}
    \mathcal B(q,\eta)
    \right)=1 .
\end{equation}
By the definition of the \((\eta,0)\)-covering number, it follows that
\begin{equation}
    \operatorname{Cov}_{\eta,0}(P')
    \le |\operatorname{supp}(Q)|
    \le \mathrm{Cov}_{\eta,\delta}(R) .
\end{equation}
Moreover, since \(\bar\eta\ge \eta\), we obtain
\begin{equation}
    \operatorname{Cov}_{\bar\eta,0}(P')
    \le
    \operatorname{Cov}_{\eta,0}(P')
    \le \mathrm{Cov}_{\eta,\delta}(R) .
\end{equation}

The second case is \(\mathcal W_{\infty,\mathrm{geo}}(R',Q)\le \eta\).

Combining this with Eq.~\eqref{eq:good_winfty_mismatch_cn} and using the triangle inequality for
the \(W_\infty\) distance, we have
\begin{equation}
    \mathcal W_{\infty,\mathrm{geo}}(P',Q)
    \le
    \mathcal W_{\infty,\mathrm{geo}}(P',R')
    +
    \mathcal W_{\infty,\mathrm{geo}}(R',Q)
    \le
    \varepsilon_\delta+\eta
    =
    \bar\eta .
\end{equation}
Thus, there exists a coupling
\(\tilde\gamma\in\Pi(P',Q)\) such that
\begin{equation}
    \mathrm{d}_{\mathrm{S}}(x,q)\le \bar\eta
    \qquad \tilde\gamma\text{-a.s.}
\end{equation}
Similarly, we obtain
\begin{equation}
    P'\left(
    \bigcup_{q\in\operatorname{supp}(Q)}
    \mathcal B(q,\bar\eta)
    \right)=1 .
\end{equation}
Hence,
\begin{equation}
    \operatorname{Cov}_{\bar\eta,0}(P')
    \le |\operatorname{supp}(Q)|
    \le \mathrm{Cov}_{\eta,\delta}(R) .
\end{equation}

Combining the two cases, we have
\begin{equation}
\label{cov_pandR}
    \operatorname{Cov}_{\bar\eta,0}(P')\le \mathrm{Cov}_{\eta,\delta}(R) .
\end{equation}
Moreover, defining
\begin{equation}
    G_P:=
    \bigcup_{q\in\operatorname{supp}(Q)}
    \mathcal B(q,\bar\eta),
\end{equation}
the previous argument shows that
\begin{equation}
    P'(G_P)=1.
\end{equation}
Since \(P=(1-2\delta)P'+2\delta P''\), it follows that
\begin{equation}
    P(G_P)\ge 1-2\delta .
\end{equation}
Consequently,
\begin{equation}
    \operatorname{Cov}_{\bar\eta,2\delta}(P)\le \mathrm{Cov}_{\eta,\delta}(R) .
\end{equation}

Consider the following matched good-component experiment:
\begin{equation}
z^*\sim P',
\qquad
u=\operatorname{sign}(Az^*+\xi'),
\end{equation}
where \(\xi'\sim\mathcal N(0,\sigma^2I_m)\), and \(\xi'\) is independent of all other random variables.
Let
\(
\widehat z
\sim
P(\cdot| u,A)
\)
be the posterior sampling estimator based on the full learned prior distribution\(P\).

By Eq.~\eqref{cov_pandR}, \(P'\) can be fully covered by at most \(\mathrm{Cov}_{\eta,\delta}(R)\)
spherical caps of radius \(\bar\eta\). Moreover, by Eq.~\eqref{eq:mixture_decomp_cn},
\[
P=(1-2\delta)P' + 2\delta P'' .
\]
Therefore, applying \(\Cref{lemma:orup}\) with full prior \(P\), good component \(P'\),
bad component \(P''\), mixture parameter \(\alpha=2\delta\), covering radius
\(\bar\eta\), we get \(N=|\operatorname{supp}(Q)|\le \mathrm{Cov}_{\eta,\delta}(R)\).Then we obtain
\begin{align}
\label{eq:matched_pprime_bound_cn}
\Pr_{z^*\sim P',A,\xi',\widehat{z}}\left(
\mathrm{d}_{\mathrm{S}}(z^*,\widehat z)
\ge
(c+1)\bar\eta
\right)
&\le
\frac{1}{1-2\delta}
\mathrm{Cov}_{\bar\eta,2\delta}(P)
\exp\!\left(
-\frac m2\Delta_\sigma(c,\bar\eta)^2
\right)
\nonumber,\\
&\le
\frac{1}{1-2\delta}
\exp\!\left(
\log \mathrm{Cov}_{\eta,\delta}(R)-\frac m2\Delta_\sigma(c,\bar\eta)^2
\right)
\nonumber,\\
&=
\exp\!\left(
\log \mathrm{Cov}_{\eta,\delta}(R)+\log\frac{1}{1-2\delta}
-\frac m2\Delta_\sigma(c,\bar\eta)^2
\right).
\end{align}

By Eq.~\eqref{eq:good_winfty_mismatch_cn}, we have
\begin{equation}
\mathcal W_{\infty,\mathrm{geo}}(R',P')
\le
\varepsilon_\delta .
\end{equation}
Therefore, there exists a coupling \(\Pi\in\Pi(R',P')\) such that, if \((x^*,z^*)\sim \Pi\),
then almost surely
\begin{equation}
\mathrm{d}_{\mathrm{S}}(x^*,z^*)\le \varepsilon_\delta .
\end{equation}

Now consider two one-bit observation models:
\begin{equation}
y=\operatorname{sign}(Ax^*+\xi),
\qquad
u=\operatorname{sign}(Az^*+\xi').
\end{equation}
Here \(x^*\sim R'\) and \(z^*\sim P'\). The posterior samples are both taken
with respect to the full learned prior distribution\(P\), namely
\begin{equation}
\widehat x\sim P(\cdot| y,A),
\qquad
\widehat z\sim P(\cdot| u,A).
\end{equation}
Applying
\Cref{lem:renyi_onebit_mismatch_explicit} with
\[
R_0=R',
\qquad
S_0=P',
\qquad
D=P,
\qquad
\epsilon=\varepsilon_\delta,
\qquad
d=(c+1)\bar\eta,
\]
we obtain
\begin{align}
\label{eq:renyi_transfer_step_cn}
&
\Pr_{x^*\sim R',A,\xi,\widehat{x}}\left(
\mathrm{d}_{\mathrm{S}}(x^*,\widehat x)
\ge
(c+1)\bar\eta+\varepsilon_\delta
\right)
\nonumber\\
&\qquad\le
\left(
1-
\frac{\lambda(1+\lambda)r_{\varepsilon_\delta}^2}{\sigma^2}
\right)^{-\frac{m}{2(1+\lambda)}}
\Pr_{z^*\sim P',A,\xi',\widehat{z}}\left(
\mathrm{d}_{\mathrm{S}}(z^*,\widehat z)
\ge
(c+1)\bar\eta
\right)^{\frac{\lambda}{1+\lambda}} .
\end{align}

Substituting Eq.~\eqref{eq:matched_pprime_bound_cn} into
Eq.~\eqref{eq:renyi_transfer_step_cn}, we get
\begin{align}
\label{eq:good_part_final_bound_cn}
&
\Pr_{x^*\sim R',A,\xi,\widehat{x}}\left(
\mathrm{d}_{\mathrm{S}}(x^*,\widehat x)
\ge
(c+1)\bar\eta+\varepsilon_\delta
\right)
\nonumber\\
&\qquad\le
\left(
1-
\frac{\lambda(1+\lambda)r_{\varepsilon_\delta}^2}{\sigma^2}
\right)^{-\frac{m}{2(1+\lambda)}}
\exp\!\left(
\frac{\lambda}{1+\lambda}
\left(
\log \mathrm{Cov}_{\eta,\delta}(R)+\log\frac{1}{1-2\delta}
-\frac m2\Delta_\sigma(c,\bar\eta)^2
\right)
\right).
\end{align}
For each realization of \(A\), define
\begin{align}
G(A)
:=
\Pr_{x^*\sim R',\xi,\widehat{x}| A}
\left(
\mathrm{d}_{\mathrm{S}}(x^*,\widehat x)
\ge
(c+1)\bar\eta+\varepsilon_\delta
\right).
\end{align}
Then Eq.~\eqref{eq:good_part_final_bound_cn} gives
\begin{align}
\mathbb E_A\left[G(A)\right]
\le
\left(
1-
\frac{\lambda(1+\lambda)r_{\varepsilon_\delta}^2}{\sigma^2}
\right)^{-\frac{m}{2(1+\lambda)}}
\exp\!\left(
\frac{\lambda}{1+\lambda}
\left(
\log \mathrm{Cov}_{\eta,\delta}(R)+\log\frac{1}{1-2\delta}
-\frac m2\Delta_\sigma(c,\bar\eta)^2
\right)
\right).
\end{align}
By Markov's inequality, for every \(\beta\in(0,1)\), we obtain that, with probability at least \(1-\beta\) over \(A\),
\begin{align}
\label{eq.173}
&
\Pr_{x^*\sim R',\xi,\widehat{x}| A}
\left(
\mathrm{d}_{\mathrm{S}}(x^*,\widehat x)
\ge
(c+1)\bar\eta+\varepsilon_\delta
\right)
\nonumber\\
&\qquad\le
\frac{1}{\beta}
\left(
1-
\frac{\lambda(1+\lambda)r_{\varepsilon_\delta}^2}{\sigma^2}
\right)^{-\frac{m}{2(1+\lambda)}}
\exp\!\left(
\frac{\lambda}{1+\lambda}
\left(
\log \mathrm{Cov}_{\eta,\delta}(R)+\log\frac{1}{1-2\delta}
-\frac m2\Delta_\sigma(c,\bar\eta)^2
\right)
\right).
\end{align}

By the decomposition Eq.~\eqref{eq:mixture_decomp_cn}, we have
\begin{equation}
R=(1-2\delta)R' + 2\delta R'' .
\end{equation}
Therefore, for every fixed realization of \(A\), the probability that the true signal \(x^*\) comes from the residual component \(R''\) is
\(2\delta\). For any event \(\mathcal E\), we have
\begin{equation}
\Pr_{x^*\sim R,\xi,\widehat{x}| A}(\mathcal E)
\le
2\delta+
\Pr_{x^*\sim R',\xi,\widehat{x}| A}(\mathcal E).
\end{equation}
Let
\begin{equation}
\mathcal E
=
\left\{
\mathrm{d}_{\mathrm{S}}(x^*,\widehat x)
\ge
(c+1)\bar\eta+\varepsilon_\delta
\right\}.
\end{equation}
Combining this with Eq.~\eqref{eq.173}, we obtain that, with probability at least \(1-\beta\) over \(A\),
\begin{align}
&
\Pr_{x^*\sim R,\xi,\widehat{x}| A}\left(
\mathrm{d}_{\mathrm{S}}(x^*,\widehat x)
\ge
(c+1)\bar\eta+\varepsilon_\delta
\right)
\nonumber\\
&\qquad\le
2\delta
+
\frac{1}{\beta}\left(
1-
\frac{\lambda(1+\lambda)r_{\varepsilon_\delta}^2}{\sigma^2}
\right)^{-\frac{m}{2(1+\lambda)}}
\exp\!\left(
\frac{\lambda}{1+\lambda}
\left(
\log \mathrm{Cov}_{\eta,\delta}(R)+\log\frac{1}{1-2\delta}
-\frac m2\Delta_\sigma(c,\bar\eta)^2
\right)
\right).
\end{align}

This completes the proof.
\end{proof}

%\uppercorollary*

\begin{proof}[Proof of Theorem~\ref{cor:main}]
Let $\Delta:=\Delta_\sigma(c,\bar\eta),r:=r_{\varepsilon_\delta}$.
By the definition of the one-bit separation gap, we have
\(\Delta_\sigma(c,\bar\eta)=f_\sigma(c\bar\eta)-f_\sigma(\bar\eta)\).
Since \(c>1\), and \(f_\sigma\) is monotonically increasing on the corresponding interval, it follows that \(\Delta>0\).
Moreover, since the range of \(f_\sigma\) is contained in \([0,1]\), we have
\[
    0<\Delta\le 1.
\]

Set
\begin{equation}
    \beta
    :=
    \exp\left(
    -\frac{m}{48}\Delta^2
    \right).
\end{equation}
By Lemma~\ref{lem:main}, for any parameter
\(0<\lambda\le
\min\left\{
1,\,
\frac{\sigma^2}{3r^2}
\right\}\),
with probability at least \(1-\beta\) over the draw of \(A\), we have
\begin{align}
&
\Pr_{x^*\sim R,\xi,\widehat{x}| A}
\left(
\mathrm{d}_{\mathrm{S}}(x^*,\widehat x)
\ge
(c+1)\bar\eta+\varepsilon_\delta
\right)
\nonumber,\\
&\qquad\le
2\delta
+
\frac{1}{\beta}
\left(
1-
\frac{\lambda(1+\lambda)r^2}{\sigma^2}
\right)^{-\frac{m}{2(1+\lambda)}}
\exp\left(
\frac{\lambda}{1+\lambda}
\left(
\log \mathrm{Cov}_{\eta,\delta}(R)+\log{\frac{1}{1-2\delta}}-\frac m2\Delta^2
\right)
\right)
\nonumber,\\
&\qquad\le
2\delta
+
\frac{1}{\beta}
\left(
1-
\frac{\lambda(1+\lambda)r^2}{\sigma^2}
\right)^{-\frac{m}{2(1+\lambda)}}
\exp\left(
\frac{\lambda}{1+\lambda}
\left(
\log \mathrm{Cov}_{\eta,\delta}(R)+\log2-\frac m2\Delta^2
\right)
\right),
\label{eq:cor_proof_start_cn}
\end{align}
where the last step uses \(\delta<1/4\), and hence
\(\log\frac{1}{1-2\delta}\le \log2\).

We now take \(\lambda=1\).
First, we verify that this choice is admissible. By the small model-mismatch assumption in the corollary, we have
\begin{equation}
    \frac{r^2}{\sigma^2}
    \le
    \frac{1}{6}\Delta^2
    \le
    \frac16.
\end{equation}
Thus
\begin{equation}
    1
    \le
    \min\left\{
    1,\,
    \frac{\sigma^2}{3r^2}
    \right\}.
\end{equation}
Hence \(\lambda=1\) lies in the range allowed by the main theorem. Substituting \(\lambda=1\) into \eqref{eq:cor_proof_start_cn} yields
\begin{align}
&
\Pr_{x^*\sim R,\xi,\widehat{x}| A}
\left(
\mathrm{d}_{\mathrm{S}}(x^*,\widehat x)
\ge
(c+1)\bar\eta+\varepsilon_\delta
\right)
\nonumber\\
&\qquad\le
2\delta
+
\frac{1}{\beta}
\left(
1-\frac{2r^2}{\sigma^2}
\right)^{-m/4}
\exp\left(
\frac12 (\log \mathrm{Cov}_{\eta,\delta}(R)+\log2)-\frac m4\Delta^2
\right).
\label{eq:lambda_one_bound_cn}
\end{align}

Next, we control the multiplicative factor induced by the model mismatch. Let
\(u:=\frac{2r^2}{\sigma^2}\).
By the assumption,
\begin{equation}
    0\le u\le \frac{\Delta^2}{3}\le \frac13.
\end{equation}
Using the elementary inequality
\begin{equation}
    -\log(1-u)\le \frac{u}{1-u}\le \frac32 u,
    \qquad 0\le u\le \frac13,
\end{equation}
we obtain
\begin{equation}
    -\log\left(1-\frac{2r^2}{\sigma^2}\right)
    \le
    3\frac{r^2}{\sigma^2}.
\end{equation}
Therefore,
\begin{align}
\label{eq:prefactor_control_cn}
\left(
1-\frac{2r^2}{\sigma^2}
\right)^{-m/4}
&=
\exp\left(
-\frac m4
\log\left(
1-\frac{2r^2}{\sigma^2}
\right)
\right)
\le
\exp\left(
\frac{3m}{4}\frac{r^2}{\sigma^2}
\right)\le
\exp\left(
\frac{m}{8}\Delta^2
\right),
\end{align}
where the last step uses
\(\frac{r^2}{\sigma^2}\le
\frac{1}{6}\Delta^2\).

Substituting \eqref{eq:prefactor_control_cn} into \eqref{eq:lambda_one_bound_cn}, we get
\begin{align}
\label{eq.242}
&
\Pr_{x^*\sim R,\xi,\widehat{x}| A}
\left(
\mathrm{d}_{\mathrm{S}}(x^*,\widehat x)
\ge
(c+1)\bar\eta+\varepsilon_\delta
\right)
\\
&\qquad\le
2\delta
+
\frac{1}{\beta}
\exp\left(
\frac12 (\log \mathrm{Cov}_{\eta,\delta}(R)+\log2)-\frac m4\Delta^2+\frac m8\Delta^2
\right)
\\
&\qquad=
2\delta
+
\frac{1}{\beta}
\exp\left(
\frac12 (\log \mathrm{Cov}_{\eta,\delta}(R)+\log2)-\frac m8\Delta^2
\right).
\end{align}

If, furthermore, \(m\ge \frac{8(\log \mathrm{Cov}_{\eta,\delta}(R)+\log2)}{\Delta^2}\),
then
\begin{equation}
    \frac12 (\log \mathrm{Cov}_{\eta,\delta}(R)+\log2)-\frac m8\Delta^2
    \le
    -\frac{m}{16}\Delta^2.
\end{equation}
Together with the choice
\(\beta=\exp\left(-\frac{m}{48}\Delta^2\right)\), this gives
\begin{equation}
    \frac{1}{\beta}
    \exp\left(
    \frac12 (\log \mathrm{Cov}_{\eta,\delta}(R)+\log2)-\frac m8\Delta^2
    \right)
    \le
    \exp\left(
    \frac{m}{48}\Delta^2-\frac{m}{16}\Delta^2
    \right)
    =
    \exp\left(
    -\frac{m}{24}\Delta^2
    \right).
\end{equation}
It follows from Eq.~\eqref{eq.242} that, with probability at least
\(1-\exp\left(-\frac{m}{48}\Delta_\sigma(c,\bar\eta)^2\right)\)
over the draw of \(A\),
\begin{equation}
    \Pr_{x^*\sim R,\xi,\widehat{x}| A}
    \left(
    \mathrm{d}_{\mathrm{S}}(x^*,\widehat x)
    \ge
    (c+1)\bar\eta+\varepsilon_\delta
    \right)
    \le
    2\delta
    +
    \exp\left(
    -\frac{m}{24}\Delta_\sigma(c,\bar\eta)^2
    \right).
\end{equation}
This proves the first probability upper bound.

We next prove the order-wise sample complexity statement. To ensure that the second term on the right-hand side and the exceptional probability over \(A\) are no larger than the prescribed
\(\zeta\in(0,1)\), it suffices that
\begin{equation}
    \exp\left(
    -\frac{m}{24}\Delta_\sigma(c,\bar\eta)^2
    \right)
    \le
    \zeta.
\end{equation}
Equivalently, it suffices that
\(m\ge24\frac{\log(1/\zeta)}{\Delta_\sigma(c,\bar\eta)^2}\).
On the other hand, to use the exponential bound proved above, we also need
\(m\ge8\frac{\log \mathrm{Cov}_{\eta,\delta}(R)+\log2}{\Delta_\sigma(c,\bar\eta)^2}\).
Therefore, it suffices that
\begin{equation}
    m
    \ge
    C
    \frac{\log \mathrm{Cov}_{\eta,\delta}(R)+\log(1/\zeta)+\log2}
    {\Delta_\sigma(c,\bar\eta)^2}
\end{equation}
where \(C>0\) is an absolute constant. For example, one may take \(C=48\). Hence, with probability at least \(1-\zeta\) over the draw of \(A\),
\begin{equation}
    \Pr_{x^*\sim R,\xi,\widehat{x}| A}
    \left(
    \mathrm{d}_{\mathrm{S}}(x^*,\widehat x)
    \ge
    (c+1)\bar\eta+\varepsilon_\delta
    \right)
    \le
    2\delta+\zeta.
\end{equation}

Finally, we analyze the order of \(\Delta_\sigma(c,\bar\eta)\) under different noise levels.
From
\begin{equation}
    f_\sigma(d)
    =
    \frac1\pi
    \arccos\left(
    \frac{\cos(\pi d)}{\sqrt{1+\sigma^2}}
    \right)
\end{equation}
direct differentiation gives
\begin{equation}
    f_\sigma'(d)
    =
    \frac{\sin(\pi d)}
    {\sqrt{\sigma^2+\sin^2(\pi d)}}.
\end{equation}
Thus
\begin{equation}
    \Delta_\sigma(c,\bar\eta)
    =
    \int_{\bar\eta}^{c\bar\eta}
    \frac{\sin(\pi t)}
    {\sqrt{\sigma^2+\sin^2(\pi t)}}\,dt.
\end{equation}
Since \(c>1\) is fixed and \(\bar\eta=\Theta(\eta)\), the integral is over an interval of length \(\Theta(\eta)\), and \(t=\Theta(\eta)\) throughout the interval. Therefore,
\[
    \sin(\pi t)\asymp t\asymp \eta.
\]
If \(\sigma\lesssim\eta\), then
\begin{equation}
    \frac{\sin(\pi t)}
    {\sqrt{\sigma^2+\sin^2(\pi t)}}
    \asymp
    1,
\end{equation}
and hence
\begin{equation}
    \Delta_\sigma(c,\bar\eta)
    \asymp
    \eta.
\end{equation}
If \(\sigma\gtrsim\eta\), then
\begin{equation}
    \frac{\sin(\pi t)}
    {\sqrt{\sigma^2+\sin^2(\pi t)}}
    \asymp
    \frac{\eta}{\sigma},
\end{equation}
and hence
\begin{equation}
    \Delta_\sigma(c,\bar\eta)
    \asymp
    \frac{\eta^2}{\sigma}.
\end{equation}
Substituting these estimates into
\(m\ge C\frac{\log \mathrm{Cov}_{\eta,\delta}(R)+\log2+\log(1/\zeta)}
{\Delta_\sigma(c,\bar\eta)^2}\)
gives
\begin{equation}
    m=
\begin{cases}
O\!\left((\log \mathrm{Cov}_{\eta,\delta}(R)+\log(1/\zeta))\,\eta^{-2}\right),
& \sigma\lesssim \eta,\\[2mm]
O\!\left((\log \mathrm{Cov}_{\eta,\delta}(R)+\log(1/\zeta))\,\sigma^2\eta^{-4}\right),
& \sigma\gtrsim \eta .
\end{cases}
\end{equation}
This completes the proof.
\end{proof}

\section{Lower bound proofs}
\label{section_proof_B}

%\oneupper*
\subsection{Proof of Lemma~\ref{lemma:oneupper}}
We prove the two claims separately, according to whether the measurement matrix
is random Gaussian or deterministic.

Case 1: When $A_{ij}\sim \mathcal N(0,1)$, we prove an upper bound on the conditional mutual information
\(I(y;x^\ast| A)\). By definition,
\begin{equation}
    I(y;x^\ast| A)
=\mathbb E_{A_0}\big[I(y;x^\ast| A=A_0)\big].
\end{equation}

Therefore, in the inner mutual information, the measurement matrix \(A\) is fixed;
the expectation over \(A\) is only taken at the very end in the outer layer.

For a fixed realization of \(A\), the observations \(y_1,\ldots,y_m\) are conditionally independent given
\(x^\ast\). Thus,
\begin{align}
I(y;x^\ast| A)
&=H(y| A)-H(y| x^\ast,A)\\
&\le\sum_{i=1}^m H(y_i| A)
-\sum_{i=1}^m H(y_i| x^\ast,A)\\
&=\sum_{i=1}^m I(y_i;x^\ast| A).
\end{align}
For each \(i\), conditioned on \(A\) and \(x^\ast\), we have
\begin{equation}
    \Pr\left(y_i=1| x^\ast,A\right)
=\Phi\left(\frac{a_i^\top x^\ast}{\sigma}\right).
\end{equation}

Denote
\begin{equation}
p_i(A,x^\ast)
:=
\Phi\left(\frac{a_i^\top x^\ast}{\sigma}\right).
\end{equation}
Since \(y_i\in\{-1,+1\}\), for a fixed \(A\), we have
\begin{equation}
I(y_i;x^\ast| A)
=
H(y_i| A)-H(y_i| x^\ast,A)
\le
\log 2-\mathbb E_{x^\ast}\big[h(p_i(A,x^\ast))| A\big],
\end{equation}
where $h(p)=-p\log p-(1-p)\log(1-p)$
is the binary entropy function.

Using the basic inequality
\begin{equation}
\log 2-h(p)\le 4\log 2\left(p-\frac12\right)^2,
\qquad p\in[0,1],
\end{equation}
we obtain
\begin{align}
I(y_i;x^\ast| A)
&\le
4\log 2\,
\mathbb E_{x^\ast}\left[
\left(
\Phi\left(\frac{a_i^\top x^\ast}{\sigma}\right)
-\frac12
\right)^2
\middle| A
\right].
\end{align}
Therefore, taking the outer expectation of the above inequality with respect to $A$ and summing over $i=1,\ldots,m$, we obtain
\begin{align}
\label{eq.226}
I(y;x^\ast| A)
&\le
4\log 2
\sum_{i=1}^m
\mathbb E_{A,x^\ast}\left[
\left(
\Phi\left(\frac{a_i^\top x^\ast}{\sigma}\right)
-\frac12
\right)^2
\right].
\end{align}
Now, the expectation is the joint expectation over both \(A\) and \(x^\ast\).
Since \(A\) is independent of \(x^\ast\), and \(\|x^\ast\|_2=1\),
for any fixed \(x^\ast\), we have
$a_i^\top x^\ast \sim \mathcal N(0,1)$.Therefore, for \(Z\sim\mathcal N(0,1)\), we have

\begin{equation}
    \mathbb E_{A,x^\ast}\left[\left(\Phi\left(\frac{a_i^\top x^\ast}{\sigma}\right)-\frac12\right)^2\right]=
\mathbb E_Z\left[\left(\Phi\left(\frac{Z}{\sigma}\right)-\frac12\right)^2\right].
\end{equation}

By Claim~\ref{claim421}, taking \(c=1/\sigma\), we obtain
\begin{equation}
\mathbb E_Z\left[\left(\Phi\left(\frac{Z}{\sigma}\right)-\frac12\right)^2\right]
=\frac{1}{2\pi}\arcsin\left(\frac{1}{1+\sigma^2}\right).
\end{equation}

Combining the above results, we obtain
\begin{equation}
    I(y;x^\ast| A)\le4\log 2
\sum_{i=1}^m\frac{1}{2\pi}\arcsin\left(\frac{1}{1+\sigma^2}\right).
\end{equation}

Thus,
\begin{equation}
I(y;x^\ast| A)
\le\frac{2m\log 2}{\pi}\arcsin\left(\frac{1}{1+\sigma^2}\right).
\end{equation}
Case 2: When $A$ is a deterministic matrix.
Since, conditional on \(x^\ast\), the observations
\(y_1,\ldots,y_m\) are independent, we have
\begin{align}
    I(y;x^\ast)
    &=
    H(y)-H(y| x^\ast) \\
    &\le
    \sum_{i=1}^m H(y_i)
    -
    \sum_{i=1}^m H(y_i| x^\ast) \\
    &=
    \sum_{i=1}^m I(y_i;x^\ast).
\end{align}
For each \(i\), we have
\begin{equation}
    \Pr\left(y_i=1| x^\ast\right)
    =
    \Phi\left(
    \frac{a_i^\top x^\ast}{\sigma}
    \right).
\end{equation}
Therefore, using the Eq.~\eqref{eq.226}, we obtain
\begin{equation}
    I(y_i;x^\ast)
    \le
    4\log 2\,
    \mathbb E_{x^\ast}
    \left[
    \left(
    \Phi\left(
    \frac{a_i^\top x^\ast}{\sigma}
    \right)
    -
    \frac12
    \right)^2
    \right].
\end{equation}
According to the Lagrange mean value theorem, we have
\begin{align}
    \left|\Phi(t) - \Phi(0)\right| &= \left|\Phi'(\xi) \cdot (t - 0)\right|,\\
    &\le   \frac{\left|t\right|}{\sqrt{2\pi}}.
\end{align}
Taking $t=\frac{a_i^\top x^\ast}{\sigma}$,we obtain
\begin{equation}
    \left(
    \Phi\left(
    \frac{a_i^\top x^\ast}{\sigma}
    \right)
    -
    \frac12
    \right)^2
    \le
    \frac{(a_i^\top x^\ast)^2}{2\pi\sigma^2}.
\end{equation}
Therefore,
\begin{equation}
    I(y_i;x^\ast)
    \le
    \frac{2\log 2}{\pi\sigma^2}
    \mathbb E_{x^\ast}
    \left[
    (a_i^\top x^\ast)^2
    \right].
\end{equation}
Summing over \(i=1,\ldots,m\), we get
\begin{align}
    I(y;x^\ast)
    &\le
    \frac{2\log 2}{\pi\sigma^2}
    \sum_{i=1}^m
    \mathbb E_{x^\ast}
    \left[
    (a_i^\top x^\ast)^2
    \right], \\
    &\le
    \frac{2\log 2}{\pi\sigma^2}
    \sum_{i=1}^m
    \|a_i\|_2^2, \\
    &\le
    \frac{2m\log 2}{\pi\sigma^2}
    \|A\|_{2,\infty}^2.
\end{align}
\begin{claim}
\label{claim421}
    For a standard normal random variable $Z \sim \mathcal{N}(0,1)$, where $\Phi(\cdot)$ is the standard normal cumulative distribution function, we have:
\begin{equation}
    \mathbb{E}_{Z}\left[\left(\Phi(cZ) - \frac{1}{2}\right)^2\right] = \frac{1}{2\pi} \arcsin\left(\frac{c^2}{1+c^2}\right) .
\end{equation}
\end{claim}
\begin{proof}
We introduce two new standard normal random variables $U_1, U_2 \sim \mathcal{N}(0,1)$ that are mutually independent and independent of $Z$.
Then, conditioned on $Z$:
\begin{equation}
    \Phi(cZ) = \Pr\left(U_1 \le cZ | Z\right) = \Pr\left(U_1 - cZ \le 0 | Z\right),
\end{equation}
\begin{equation}
\Phi(cZ)^2 = \Pr\left(U_1 \le cZ, U_2 \le cZ | Z\right) = \Pr\left(U_1 - cZ \le 0, U_2 - cZ \le 0 | Z\right).
\end{equation}
Define the random variables: $W_1 = U_1 - cZ$, $W_2 = U_2 - cZ$. Since $U_1, U_2, Z\sim\mathcal N(0,1)$, $W_1$ and $W_2$ follow a joint zero-mean bivariate normal distribution.
By the additivity of Gaussian distributions, we have
\begin{equation}
    \mathrm{Var}(W_1) = \mathrm{Var}(U_1) + c^2\mathrm{Var}(Z) = 1 + c^2.
\end{equation}
\begin{equation}
    \mathrm{Cov}(W_1, W_2) = \mathbb{E}[(U_1 - cZ)(U_2 - cZ)] = c^2 \mathbb{E}[Z^2] = c^2.
\end{equation}
    Therefore, the correlation coefficient between $W_1$ and $W_2$ is:
\begin{align}
    \rho &= \frac{\mathrm{Cov}(W_1, W_2)}{\sqrt{\mathrm{Var}(W_1)\mathrm{Var}(W_2)}},\\
\label{Eq.26}
    &= \frac{c^2}{1+c^2}.
\end{align}
For joint bivariate normal variables with mean $0$ and correlation coefficient $\rho$, the probability that both are simultaneously less than or equal to $0$ is strictly given by the orthant probability formula for the standard bivariate normal distribution~\cite{o2014analysis}:
\begin{equation}
\label{Eq.22}
    \Pr\left(W_1 \le 0, W_2 \le 0\right) = \frac{1}{4} + \frac{1}{2\pi} \arcsin(\rho).
\end{equation}
Combining the above equations, we have:
\begin{align}
     \mathbb{E}_{Z}\left[\left(\Phi(cZ) - \frac{1}{2}\right)^2\right] &= \mathbb{E}_{Z}[\Phi(cZ)^2] - \mathbb{E}_{Z}[\Phi(cZ)] + \frac{1}{4},\\
     &=\frac{1}{4} + \frac{1}{2\pi} \arcsin(\rho)-\frac{1}{2}+ \frac{1}{4},\\
     &=\frac{1}{2\pi} \arcsin(\frac{c^2}{1+c^2}).
\end{align}
\end{proof}
%\end{proof}

\begin{lemma}[Lemma 4.3 in~\citep{jalal2021instance}]%{lemma}{lammaB3}
\label{lemma:4.3}
Consider the setting of \Cref{thm:one_bit_sphere_lower_bound}. If the measurement matrix $A$ is deterministic, we have:
\begin{equation}
    I(x^*; \widehat{x}) \le I(y; x^*).
\end{equation}

If $A$ is a Gaussian matrix, then:
\begin{equation}
    I(x^*; \widehat{x}) \le I(y; x^*|A).
\end{equation}
\end{lemma}

\begin{lemma}[Adapted from~\citep{jalal2021instance}]
\label{lemma:B3_metric}
Let $Q$ be a distribution defined on a finite discrete set $S$ containing $N$ points. Assume there exists a sufficiently large integer $N_2$ such that for each $u$ in the support, $Q(u) = j \alpha$, where $j \in \mathbb{N}$, and the quantization step is $\alpha := 1/N_2$. 
Let $(x, \widehat{x})$ be jointly distributed, where $x \sim Q$, and $\widehat{x}$ lies in an arbitrary countable set. If the recovery condition is satisfied:
\begin{equation}
    \Pr\left(d(x, \widehat{x}) \le \epsilon\right) \ge 1-\delta.
\end{equation}
Then we have:
\begin{equation}
    \tau(1-\delta)\log \mathrm{Cov}_{2\epsilon, \tau+\delta}(Q) \le I(x; \widehat{x}) + 2.
\end{equation}
\end{lemma}

%\spherefano*

\subsection{Proof of Lemma~\ref{lemma:spherefano_new}}

Let $E=1_{\mathrm{d}_{\mathrm{S}}(x,\widehat{x})\le \epsilon}$ denote the indicator variable.

Let $\epsilon = \eta$. Since the continuous unit sphere $\mathbb{S}^{n-1}$ is a compact manifold, we construct a finite spherical $\gamma$-net $\mathcal{N}_\gamma$ with resolution $\gamma$ ($\gamma >0$) on it. We map the continuous signal $x$ and the recovered result $\widehat{x}$ to their nearest discrete points $\bar{x}$ and $\widehat{\bar{x}}$ on $\mathcal{N}_\gamma$.
By bounding, we obtain the following discretization error:
\begin{align}
    d_{{S}}(\bar{x}, \widehat{\bar{x}}) &\le d_{{S}}(\bar{x}, x) + d_{S}(x, \widehat{x}) + \mathrm{d}_{\mathrm{S}}(\widehat{x}, \widehat{\bar{x}}), \\
    &\le \gamma + \epsilon + \gamma, \\
    &= \epsilon + 2\gamma.
\end{align}
Therefore, the high-probability recovery event is transformed into $\Pr\left(\mathrm{d}_{\mathrm{S}}(\bar{x}, \widehat{\bar{x}}) \le \epsilon + 2\gamma\right) \ge 1-\delta$.

Since $\bar{x}$ and $\widehat{\bar{x}}$ are deterministic functions of $x$ and $\widehat{x}$ respectively, according to \Cref{lemma:4.3}, the data processing does not increase the mutual information, so we have:
\begin{equation}
    I(\bar{x}; \widehat{\bar{x}}) \le I(x; \widehat{x}).
\end{equation}

Let $\bar{x}$ follow the discrete distribution $\bar{R}$. Since $\mathcal{N}_\gamma$ is a finite set, the support of $\bar{R}$ is naturally bounded. To satisfy the probability quantization requirement of \Cref{lemma:B3_metric}, we introduce a minimum quantization step size $\alpha$, and quantize the probability values in $\bar{R}$ downwards to integer multiples of $\alpha$ to construct a finite discrete distribution $Q'$. The total probability loss in this step does not exceed $\alpha |\mathcal{N}_\gamma|$. We choose $\alpha$ sufficiently small such that $\alpha |\mathcal{N}_\gamma| \le \delta$.

Given the concavity of mutual information with respect to the marginal distribution under a fixed conditional distribution~\cite{cover1999elements}, the adjusted mutual information satisfies:
\begin{equation}
\label{eq.26}
    I(\bar{x}; \widehat{\bar{x}}) \ge (1-\delta)I_{Q'}(\bar{x}; \widehat{\bar{x}}).
\end{equation}
Meanwhile, under the distribution $Q'$, the lower bound of the successful recovery probability decreases to $1 - \delta - \delta = 1 - 2\delta$.

Now $Q'$ is a finite distribution that strictly satisfies the probability quantization. Applying  \Cref{lemma:B3_metric} in the metric space, we have:
\begin{equation}
    \tau(1-2\delta) \log \mathrm{Cov}_{2(\epsilon+2\gamma), \tau+2\delta}(Q') \le I_{Q'}(\bar{x}; \widehat{\bar{x}}) + 2.
\end{equation}
Substituting Eq.~\eqref{eq.26} into the above equation yields:
\begin{equation}
    \tau(1-2\delta) \log \mathrm{Cov}_{2\epsilon+4\gamma, \tau+2\delta}(Q') \le \frac{I(x; \widehat{x})}{1-\delta} + 2.
\end{equation}

Furthermore, we need to generalize the approximate covering bound established on the discrete distribution $Q'$ to the true continuous prior distribution $R$. Considering that $Q'$ truncates at most a probability measure of $\delta$ relative to $R$, this means that any subset capable of covering a measure of $1 - (\tau + 2\delta)$ in $Q'$ will necessarily cover a measure of at least $1 - (\tau + 3\delta)$ in $R$. At the same time, to accommodate the geometric error introduced by spatial discretization, we relax the bound of the covering radius by $\gamma$ to $2\epsilon + 5\gamma$. At this point, if we choose the grid resolution $\gamma = \epsilon/5$, the new covering radius corresponds exactly to $3\epsilon = 3\eta$. Substituting the above conclusions and multiplying both sides of the inequality by $(1-\delta)$, we can prove:
\begin{equation}
    (1-\delta)\tau(1-2\delta)\log \mathrm{Cov}_{3\eta, \tau+3\delta}(R) \le I(x; \widehat{x}) + 2(1-\delta).
\end{equation}

%\onebitspherethm*

\subsection{Proof of Theorem~\ref{thm:one_bit_sphere_lower_bound}}
Let $N(R, \gamma)$ denote the minimal set of $3\eta$-geodesic balls covering at least $1-\gamma$ probability mass of $R$, such that its cardinality is $\mathrm{Cov}_{3\eta, \gamma}(R)$. Let $B$ be the ball with the smallest marginal probability in $N(R, 10\delta)$. Define the subset $S = N(R, 10\delta) \setminus B$, which strictly covers less than $1-10\delta$ of the probability mass. We decompose the distribution $R$ into a mixture distribution $R = (1-c)R' + cR''$, where $R'$ and $R''$ are the conditional distributions restricted to $S$ and its complement $S^c$, respectively. By the construction of $S$, we have $c > 10\delta$. Since $R''$ contributes at least $10\delta$ of the mass to $R$, any algorithm that succeeds with probability $\ge 1-\delta$ over the entire distribution $R$ must succeed with probability $\ge 0.9$ on $R''$.

We now apply the variant of the spherical Fano inequality ( \Cref{lemma:spherefano_new}) to the distribution $R''$. Substituting the free parameter $\tau = 0.1$ and the failure probability $\delta_{R''} = 0.1$ into  \Cref{lemma:spherefano_new} yields:
\begin{equation}
    (1 - 0.1)(0.1)(1 - 0.2) \log \mathrm{Cov}_{3\eta, 0.1 + 3(0.1)}(R'') \le I(x^*; \widehat{x}) + 2(1 - 0.1).
\end{equation}
Simplifying this inequality yields:
\begin{equation}
    I(x^*; \widehat{x}) \ge 0.072 \log \mathrm{Cov}_{3\eta, 0.4}(R'') - 1.8. \label{eq:pre_bound}
\end{equation}

Next, we strictly relate the covering number of $R''$ back to the original distribution $R$, without loose relaxation. Since $R''$ contains at least $10\delta$ of the mass of $R$, the set $N(R'', 0.4)$ which covers $\frac{3}{5}$ of the mass in $R''$, essentially covers at least $6\delta$ of the mass of $R$. Meanwhile, $N(R, 10\delta)$ covers $1-10\delta$ of the mass of $R$. The union of these two sets must cover at least $1 - 10\delta + 6\delta = 1 - 4\delta$ of the mass of $R$. Therefore, we establish the cardinality bound:
\begin{equation}
    |N(R'', 0.4)| + |N(R, 10\delta)| \ge |N(R, 4\delta)| \implies |N(R'', 0.4)| \ge |N(R, 4\delta)| - |N(R, 10\delta)|. \label{eq:union_bound}
\end{equation}

Furthermore, we use a counting argument to strictly relate $|N(R, 4\delta)|$ and $|N(R, 10\delta)|$. Assuming the balls in $N(R, 4\delta)$ are sorted in descending order of their marginal probabilities, the bottom proportion of $\frac{10\delta - 4\delta}{1-4\delta} = \frac{6\delta}{1-4\delta}$ balls can contain at most $10\delta - 4\delta = 6\delta$ of the mass (relative to the total of $1-4\delta$). Alternatively, observing that the top proportion of $\frac{1-10\delta}{1-4\delta}$ balls in $N(R, 4\delta)$ must contain at least $1-10\delta$ of the mass, we strictly obtain:
\begin{equation}
    \frac{1-10\delta}{1-4\delta} |N(R, 4\delta)| \ge |N(R, 10\delta)|. \label{eq:counting}
\end{equation}

Substituting Eq.~\eqref{eq:counting} into Eq.~\eqref{eq:union_bound}, we obtain:
\begin{align}
    |N(R'', 0.4)| &\ge |N(R, 4\delta)| - \frac{1-10\delta}{1-4\delta}|N(R, 4\delta)|, \\
                  &= \left( \frac{1-4\delta - (1-10\delta)}{1-4\delta} \right) |N(R, 4\delta)| ,\\
                  &\ge 6\delta|N(R, 4\delta)|.
\end{align}
Taking the logarithm of both sides and using the definition $|N(R, \gamma)| = \mathrm{Cov}_{3\eta, \gamma}(R)$, we obtain the tight geometric relation:
\begin{equation}
    \log \mathrm{Cov}_{3\eta, 0.4}(R'') \ge \log \mathrm{Cov}_{3\eta, 4\delta}(R) + \log(6\delta). \label{eq:cov_relation}
\end{equation}

Substituting the tight relation Eq.~\eqref{eq:cov_relation} into Eq.~\eqref{eq:pre_bound} yields:
\begin{equation}
    I(x^*; \widehat{x}) \ge 0.072 \left( \log \mathrm{Cov}_{3\eta, 4\delta}(R) + \log(6\delta) \right) - 1.8.
\end{equation}
Now, we establish the upper bound on mutual information. According to the Markov chain $x^* \to y \to \widehat{x}$ and the data processing inequality ( \Cref{lemma:4.3}), for a Gaussian matrix $A$, we have $I(x^*; \widehat{x}) \le I(y; x^*|A)$ (for a deterministic matrix $A$, we have $I(x^*; \widehat{x}) \le I(y; x^*)$). Utilizing the bound on the one-bit quantized channel capacity from \Cref{lemma:oneupper}, we obtain:

Case 1: When $A_{ij}\sim\mathcal N(0,1)$, we have
\begin{equation}
     \frac{2m \log 2}{\pi} \arcsin\left(\frac{1}{1+\sigma^2}\right)\ge 0.072 \left( \log \mathrm{Cov}_{3\eta, 4\delta}(R) + \log(6\delta) \right) - 1.8.
\end{equation}
Dividing both sides by $\frac{2 \log 2}{\pi} \arcsin\left(\frac{1}{1+\sigma^2}\right)$, we conclude:
\begin{equation}
    m \ge \frac{0.072\pi}{2(\log 2)\arcsin\left(\frac{1}{1+\sigma^2}\right)} \left( \log \mathrm{Cov}_{3\eta, 4\delta}(R) + \log(6\delta)-25 \right).
\end{equation}
Case 2: When $A$ is a deterministic matrix, we have:
\begin{equation}
     \frac{2m\log 2}{\pi \sigma^2} \|A\|_{2,\infty}^2\ge0.072 \left( \log \mathrm{Cov}_{3\eta, 4\delta}(R) + \log(6\delta) \right) - 1.8.
\end{equation}
Dividing both sides by $\frac{2\log 2}{\pi \sigma^2} \|A\|_{2,\infty}^2$, we conclude:
\begin{equation}
     m \ge \frac{0.072\pi\sigma^2}{2(\log 2)\|A\|_{2,\infty}^2} \left( \log \mathrm{Cov}_{3\eta, 4\delta}(R) + \log(6\delta)-25 \right).
\end{equation}
This completes the proof.

\section{Detailed PnP-OneBit}
\label{appendix: DetailePnP-OneBit}

\subsection{Preliminary for diffusion}
\label{appendix: Preliminary4diffusion}
By employing a continuous-time stochastic process, diffusion models construct a generative paradigm that systematically degrades a complex data distribution into a manageable prior. Suppose an uncorrupted data point $x_0 \in \mathbb{R}^n$ is sampled from the target distribution $p_{\text{data}}$. The forward corruption mechanism over a continuous timeline $t \in [0, T]$ is defined by the following stochastic differential equation (SDE):
\begin{equation}\label{eq:forward_sde}
    d x_t = f(t) x_t d t + g(t) d w_t,
\end{equation}
where $x_t \in \mathbb{R}^n$ is the state at time $t$, $w_t \in \mathbb{R}^n$ represents the standard Wiener process, $f(t)$ serves as the drift factor, and $g(t)$ denotes the diffusion coefficient. For any given $t$, the transition probability from the original sample $x_0$ to the perturbed state $x_t$ is characterized in~\citep{ho2020denoising} as:
\begin{equation}\label{eq:marginal_dist}
    q(x_t | x_0) = \mathcal{N}(x_t; \alpha_t x_0, \sigma_t^2 I_n),
\end{equation}
where the terms $\alpha_t$ and $\sigma_t$ denote differentiable schedule functions for scale and noise, with $\alpha_t$ monotonically decreasing and $\sigma_t$ monotonically increasing as $t$ progresses. To guarantee that the SDE presented in Eq.~\eqref{eq:forward_sde} accurately produces the marginal distribution $q(x_t | x_0)$, the drift and diffusion terms are tied to the schedule functions as follows~\cite{lu2022dpmsolver}:
\begin{equation}
    \label{eq:define_f_n_g}
    f(t)=\frac{d \log \alpha_t}{d  t},\quad g^2(t)=\frac{d  \sigma_t^2}{d  t} - 2\sigma_t^2\frac{d \log \alpha_t}{d  t}.
\end{equation}
At the final time $T$, the boundary condition dictates that the distribution closely mirrors a standard normal distribution, where $x_T \sim \mathcal{N}(\mathbf{0}, I_n)$.

To reverse this diffusion process and reconstruct the target data, a parameterized neural network $\epsilon_\theta(x_t, t)$ is optimized to estimate the additive noise $\epsilon \sim \mathcal{N}({0}, I_n)$. The training procedure minimizes the subsequent loss function:
\begin{equation}\label{eq:training_objective}
     \mathbb{E}_{t\sim \mathcal{U}(0,T), x_0\sim p_{\text{data}}, \epsilon\sim\mathcal{N}(0,I_d)} \left[ \| \epsilon_\theta(\alpha_tx_0 + \sigma_t\epsilon, t) - \epsilon \|_2^2 \right].
\end{equation}
According to Song et al.~\cite{songscore}, the generative process running backward from $t=T$ to $t=0$ corresponds to the following reverse-time SDE:
\begin{equation}
    \label{eq:inverSDE}
    d x_t = \left(f(t)x_t - g^2(t)\nabla_{x_t}\log p_t(x_t)\right)d t + g(t)d \overline{w}_t,
\end{equation}
where $\overline{w}_t$ indicates a backward-in-time Wiener process, and $p_t$ represents the marginal density of $x_t$. The required score function $\nabla_{x_t}\log p_t(x_t)$ is empirically estimated using the trained noise predictor via $\nabla_{x_t}\log p_t(x_t) \approx -\epsilon_\theta(x_t, t)/\sigma_t$. Additionally, an equivalent probability flow ordinary differential equation (ODE) exists, which strictly maintains identical marginal distributions to the SDE~\cite{songscore}:
\begin{equation}
    \label{eq:inverODE}
    d x_t = \left(f(t)x_t - \frac{1}{2}g^2(t)\nabla_{x_t}\log p_t(x_t)\right)d t.
\end{equation}
For practical implementation, Eq.~\eqref{eq:inverODE} is evaluated via numerical integration across a predefined discrete schedule $\{t_i\}_{i=0}^N$, constrained by $t_i>t_{i+1}$, starting at $t_0 = T$ and ending at $t_N = 0$. At a specific step $t_i$, the expected clean data can be approximated using Tweedie's formula~\cite{efron2011tweedie} as follows:
\begin{equation}\label{eq:tweedie_est}
    \hat{x}_{0|t_i} = \cfrac{x_{t_i} + \sigma_{t_i}^2\nabla_{x_{t_i}}\log p_{t_i}\left(x_{t_i}\right)}{\alpha_{t_i}}\approx\cfrac{x_{t_i} - \sigma_{t_i} \epsilon_\theta(x_{t_i}, t_i)}{\alpha_{t_i}}.
\end{equation}
Finally, the update step from $t_i$ to $t_{i+1}$ operates through the DDIM~\cite{song2020ddim} integration scheme with sampling noise level $\eta_{\text{DDIM}}$:
\begin{equation}\label{eq:ddim_step}
    x_{t_{i+1}} = \alpha_{t_{i+1}} \hat{x}_{0|t_i} + \sqrt{1-\eta_{\text{DDIM}}}\sigma_{t_{i+1}}\epsilon_\theta(x_{t_i}, t_i) + \eta_{\text{DDIM}} \xi_i,\ \xi_i\sim\mathcal{N}(0,I_n). 
\end{equation}

\subsection{Align between likelihood and prior}
\label{appendix: AlignBetweenLikelihoodNPrior}
Recall from \Cref{PnP} that the prior sampling step requires drawing $x^{(k)}$ from the following conditional distribution:
\begin{equation}
    \label{eq: sample_x_app}
    x^{(k)} \sim \Pi(x|z^{(k)}) \propto \exp\left( -\mathcal{P}(x) - \frac{1}{2\varrho^2}\|x-z^{(k)}\|_2^2 \right),
\end{equation}
where $\mathcal{P}(x)=-\log p(x)$. 

According to the diffusion formulation detailed in \Cref{appendix: Preliminary4diffusion}, the posterior distribution of the clean data $x_0$ given the noisied data $x_t$ can be expressed via Bayes' theorem as:
\begin{equation}
\label{eq: priorWithDiffusion}
\begin{aligned}
    p(x_0|x_t)&\propto p(x_t|x_0)p(x_0)\\
    &\propto\exp\left(\log p(x_0)+\log p(x_t|x_0)\right)\\
    &\propto\exp\left(-\mathcal{P}(x_0)-\frac{1}{2\sigma_t^2}\|x_t-\alpha_tx_0\|_2^2\right)\\
    &\propto\exp\left(-\mathcal{P}(x_0)-\frac{\alpha_t^2}{2\sigma_t^2}\left\|x_0-{x_t}/{\alpha_t}\right\|_2^2\right).
\end{aligned}
\end{equation}

By comparing Eq.~\eqref{eq: priorWithDiffusion} with Eq.~\eqref{eq: sample_x_app}, it becomes evident that the two distributions are mathematically equivalent under the substitutions $x=x_0$, $\varrho={\sigma_t}/{\alpha_t}$, and $z^{(k)}=x_t/\alpha_t$. Consequently, drawing a sample from Eq.~\eqref{eq: sample_x_app} is structurally identical to sampling $x_0$ from $p(x_0|x_t)$. While the exact posterior $p(x_0|x_t)$ is intractable in practice, a pretrained diffusion model provides an effective approximation of this reverse transition. Therefore, we can reliably employ the diffusion model to perform the sampling step for Eq.~\eqref{eq: sample_x_app} by setting the initial state as $x_t=\alpha_tz^{(k)}$ and matching the noise schedule such that $\varrho={\sigma_t}/{\alpha_t}$.

\subsection{Detailed algorithm}
\label{appendix: DetailedAlgorithm}

In this section, we first introduce the formulation of the measurement likelihood $\Pr\left(y|x,A\right)$ and the gradient $\nabla_{x}\mathcal{L}(x;y)$, followed by the detailed algorithmic workflow of our PnP-OneBit.

Given the one-bit observation model $A^T=[a_1^T,\dots,a_m^T]$, the probability of acquiring a single measurement $y_i$ can be derived as:
\begin{equation}
    \Pr\left(y_i|x,A\right)=\Pr\left(y_i(a_i^Tx+\sigma\zeta_i)\geq 0\right)=\Pr\left(\zeta_i\leq \frac{y_ia_i^Tx}{\sigma}\right)=\Phi\left(\frac{y_ia_i^Tx}{\sigma}\right),
\end{equation}
where $\zeta_i\sim\mathcal{N}(0,1)$ is the standard Gaussian noise, and $\Phi(\cdot)$ represents the cumulative distribution function of the standard Gaussian distribution. Assuming conditional independence among the measurements, the joint likelihood is given by:
\begin{equation}
    \Pr\left(y|x,A\right)=\prod_{i=1}^m\Phi\left(\frac{y_ia_i^Tx}{\sigma}\right).
\end{equation}
Consequently, the negative log-likelihood $\mathcal{L}(x;y)$ can be rewritten as:
\begin{equation}
    \mathcal{L}(x;y)=-\log \Pr\left(y|x,A\right)=-\sum_{i=1}^m\log \Phi\left(\frac{y_ia_i^Tx}{\sigma}\right).
\end{equation}
Taking the derivative with respect to the signal $x$ yields the analytical gradient:
\begin{equation}
\label{eq: gradOfL}
    \nabla_{x}\mathcal{L}(x;y)=- A^T \left( \frac{y}{\sigma} \odot \frac{\phi(r)}{\Phi(r)}\right),
\end{equation}
where $r = \frac{y \odot (Ax)}{\sigma}$ and the operator $\odot$ denotes the element-wise Hadamard product. By utilizing the explicit gradient in Eq.~\eqref{eq: gradOfL}, we can apply Markov chain Monte Carlo (MCMC) methods in PnP-DM~\cite{wu2024principled} to sample from the following likelihood-based conditional distribution:
\begin{equation}
    z^{(k)} \sim \Pi(z|x^{(k-1)}) \propto \exp\left( -\mathcal{L}(z;y) - \frac{1}{2\varrho^2}\|z-x^{(k-1)}\|_2^2 \right).
\end{equation}
The complete procedure of our proposed approach is summarized in \Cref{alg: pnp_onebit}.

\begin{algorithm}[t]
\caption{PnP-OneBit}
\label{alg: pnp_onebit}
\begin{algorithmic}[1]
\Require Pretrained noise predictor model $\epsilon_\theta$, iteration schedule $\{\varrho_k\}_{k=0}^{K-1}$, diffusion schedule $\{t_i\}_{t=0}^{T}$, learning rate $\kappa$, MCMC step $J$, DDIM noise level $\eta$, measurement $y$, forward matrix $A$

\State Sample $x^{(0)} \sim \mathcal{N}({0}, I_n)$

\For{$k=0,1,\dots,K-1$}
    \State $z_0 \gets x^{(k)}$
    \For{$j=0,1,\dots,J$}
        \State $g_j \gets \nabla_{z_j}\mathcal{L}(z_j,y)$
        \State Sample ${\xi}_j \sim \mathcal{N}({0}, I_n)$
        \State $z_{j+1} \gets z_j - \kappa g_j - \frac{\kappa}{\varrho_k^2}(z_j-x^{(k)})+\sqrt{2\kappa}{\xi}_j$
    \EndFor
    \State $z^{(k)}\gets z_J$
    \State Solve $\varrho_k=\sigma_{t^\star}/\alpha_{t^\star}$ to get $t^\star$
    \State $l\gets\min\{i| t^\star>t_i\}$ \Comment{seek $l$ satisfying $t_{l-1}\geq t^\star> t_{l}$}
    \State $x_{t^\star}\gets \alpha_{t_\star} z^{(k)}$
    \State $\hat\epsilon^\star=\epsilon_\theta(x^\star, t^\star)$
    \State $\hat{x}_{0|t^\star} = \frac{x_{t^\star} - \sigma_{t^\star}\,\epsilon_\theta(x_{t^\star}, t^\star)}{\alpha_{t^\star}}$
    \State Sample $\xi^\star \sim \mathcal{N}(0, I_n)$
    \State $x_{t_{l}} = \alpha_{t_{l}} \hat{x}_{0|t^\star} + \sqrt{1-\eta_{\text{DDIM}}^2}\,\sigma_{t_{l}}\,\hat{\epsilon}^\star + \eta_{\text{DDIM}}\,\sigma_{t_{l}}\xi^\star$ \Comment{DDIM sample to $t_l$ from $t^\star$}
    \For{$i=l,l+1,\dots,T-1$}
        \State $\hat\epsilon_i=\epsilon_\theta(x_{t_i}, t_i)$
        \State $\hat{x}_{0|t_i} = \frac{x_{t_i} - \sigma_{t_i}\,\epsilon_\theta(x_{t_i}, t_i)}{\alpha_{t_i}}$
        \State Sample $\xi_i \sim \mathcal{N}(0, I_n)$
        \State $x_{t_{i+1}} = \alpha_{t_{i+1}} \hat{x}_{0|t_i} + \sqrt{1-\eta_{\text{DDIM}}^2}\,\sigma_{t_{i+1}}\,\hat\epsilon_i + \eta_{\text{DDIM}}\,\sigma_{t_{i+1}}\,\xi_i$
        \Statex \Comment{DDIM sample to $t_{i+1}$ from $t_{i}$}
    \EndFor
    \State $x^{(k+1)}\gets x_T$
\EndFor

\State \Return $x^{(K)}$

\end{algorithmic}
\end{algorithm}

\section{Experimental results}
\label{appendix: ExperimentalResult}

\subsection{Implementation details}
\label{appendix: ImplementationDetail}
Our experiments are conducted on the FFHQ~\cite{karras2019style} and ImageNet~\cite{russakovsky2015imagenet} datasets. FFHQ is licensed under Creative Commons BY-NC-SA 4.0, and ImageNet is used under its standard research license. We employ the pre-trained diffusion models provided by Dhariwal \& Nichol~\cite{dhariwal2021diffusion} and Choi et al.~\cite{choi2021ilvr}, both of which are released under the MIT License. We maintain their original noise schedules $\alpha_t$ and $\sigma_t$. Following the settings in PnP-DM, the coupling parameter $\varrho_k$ is updated via $\varrho_k=\max\{\nu^k\varrho_0,\varrho_{\min}\}$, where $\varrho_0=10$, $\nu=0.9$, and $\varrho_{\min}=0.3$. The hyperparameters for the iterative process are set as follows: MCMC steps $J=100$, learning rate $\kappa=0.01$, diffusion timesteps $T=100$, DDIM noise level $\eta=0.5$, and outer iterations $K=100$. For the baseline methods, including DiffPIR~\cite{zhu2023denoising}, DPS~\cite{chung2023diffusion}, DAPS~\cite{zhang2025improving}, QCS-SGM~\cite{mengquantized}, SIM-DMIS~\cite{tang2025learning}, and Diff-OneBit~\cite{chen2026diffusion}, we use the default parameter settings provided in their respective publications. All experiments were conducted on a single Nvidia RTX 4090 (24GB).

\subsection{Quantitative results}
\label{appendix: QuantitativeResult}
We evaluate the robustness of the algorithms on the FFHQ dataset under a higher noise level of $\sigma=1.0$ at compression ratios $n/m=16$ and $n/m=32$. As detailed in \Cref{tab: QuantitativeResultsFFHQWith1.0Sigma}, the proposed algorithm consistently maintains the highest performance under increased noise corruption. Furthermore, \Cref{tab: QuantitativeResultsImageNet} presents the evaluation results on the ImageNet dataset. These results confirm that PnP-OneBit effectively recovers one-bit signals not only for domain-specific distributions such as human faces but also for diverse, complex natural scenes.
In addition, compared to Diff-OneBit~\cite{chen2026diffusion}, our method exhibits a slightly larger standard deviation. This numerical variance originates from the fundamental difference in the problem formulation. Diff-OneBit employs deterministic optimization techniques to find a single point estimate, which mathematically suppresses variance but can lead to overly smoothed textures. In contrast, PnP-OneBit functions strictly as a posterior sampler. The stochastic transitions used to explore the target distribution inherently introduce algorithmic randomness. However, this distributional modeling approach effectively preserves fine high-frequency details, which contributes to the overall improvement in reconstruction quality.

\begin{table}[htbp]
\centering
\caption{Quantitative results of one bit recovery on FFHQ dataset with different $m$ and the additive Gaussian noise of $\sigma=1.0$. The best results are \textbf{bold}.}
\label{tab: QuantitativeResultsFFHQWith1.0Sigma}
\setlength{\tabcolsep}{3.5pt}
\begin{tabular}{lcccccc}
\toprule
\multirow{2}{*}{Methods}  & \multicolumn{3}{c}{$n/m=16$}                & \multicolumn{3}{c}{$n/m=32$}                \\ 
\cmidrule(lr){2-4} \cmidrule(lr){5-7}
                         & PSNR$\uparrow$ & SSIM$\uparrow$ & LPIPS$\downarrow$ & PSNR$\uparrow$ & SSIM$\uparrow$ & LPIPS$\downarrow$  \\ 
\hline
DiffPIR                      & 11.50$\pm$1.13          & 0.24$\pm$.05          & 0.78$\pm$.06           & 11.20$\pm$0.94          &  0.18$\pm$.05         &  0.80$\pm$.06          \\
DPS                          & 14.53$\pm$1.23          & 0.51$\pm$.06          & 0.52$\pm$.07           & 13.68$\pm$0.78          &  0.49$\pm$.06         &  0.57$\pm$.05          \\
DAPS                         & 14.50$\pm$1.48          & 0.40$\pm$.05          & 0.61$\pm$.07           & 13.57$\pm$0.98          & 0.48$\pm$.06          & 0.56$\pm$.07          \\
QCS-SGM                      & 18.43$\pm$2.47          & 0.52$\pm$.10          & 0.54$\pm$.07           & 17.00$\pm$1.76          &  0.49$\pm$.09         &  0.61$\pm$.06          \\
SIM-DMIS                     & 18.32$\pm$1.79          & 0.51$\pm$.08          & 0.53$\pm$.06           & 16.89$\pm$1.17          &  0.47$\pm$.08         &  0.59$\pm$.08          \\
Diff-OneBit                  & 21.15$\pm$1.17          & 0.40$\pm$.06          & 0.49$\pm$.06           & 18.32$\pm$0.92          &  0.48$\pm$.06         &  0.54$\pm$.06          \\
\rowcolor{gray!20}PnP-OneBit & \textbf{22.22}$\pm$1.73 & \textbf{0.58}$\pm$.09 & \textbf{0.35}$\pm$.07  & \textbf{19.37}$\pm$1.18 & \textbf{0.57}$\pm$.07 & \textbf{0.49}$\pm$.07  \\
\bottomrule
\end{tabular}
\end{table}

\begin{table}[htbp]
\centering
\caption{Quantitative results of one bit recovery on ImageNet dataset with different $m$ and the additive Gaussian noise of $\sigma=0.5$. The best results are \textbf{bold}.}
\label{tab: QuantitativeResultsImageNet}
\setlength{\tabcolsep}{3.5pt}
\begin{tabular}{lcccccc}
\toprule
\multirow{2}{*}{Methods}  & \multicolumn{3}{c}{$n/m=16$}                & \multicolumn{3}{c}{$n/m=32$}                \\ 
\cmidrule(lr){2-4} \cmidrule(lr){5-7}
                         & PSNR$\uparrow$ & SSIM$\uparrow$ & LPIPS$\downarrow$ & PSNR$\uparrow$ & SSIM$\uparrow$ & LPIPS$\downarrow$  \\ 
\hline
DiffPIR                      & 10.73$\pm$0.69          & 0.09$\pm$.03          & 0.77$\pm$.07           & 9.80$\pm$0.54          &  0.06$\pm$.02         &  0.80$\pm$.07          \\
DPS                          & 14.38$\pm$1.36          & 0.38$\pm$.05          & 0.69$\pm$.08           & 13.14$\pm$0.76          &  0.35$\pm$.04         &  0.76$\pm$.08          \\
DAPS                         & 14.17$\pm$1.51          & 0.24$\pm$.04          & 0.67$\pm$.07           & 12.27$\pm$0.77          &  0.33$\pm$.05         &  0.78$\pm$.08          \\
QCS-SGM                      & 16.28$\pm$1.92          & 0.39$\pm$.08          & 0.66$\pm$.08           & 14.13$\pm$1.64          &  0.37$\pm$.08         &  0.72$\pm$.09          \\
SIM-DMIS                     & 17.03$\pm$1.23          & 0.41$\pm$.06          & 0.60$\pm$.07           & 15.42$\pm$0.98          &  0.40$\pm$.06         &  0.68$\pm$.08          \\
Diff-OneBit                  & 18.27$\pm$0.87          & 0.43$\pm$.05          & 0.53$\pm$.05           & 16.46$\pm$0.61          &  0.42$\pm$.05         &  0.67$\pm$.07          \\
\rowcolor{gray!20}PnP-OneBit & \textbf{19.77}$\pm$1.32 & \textbf{0.53}$\pm$.07 & \textbf{0.47}$\pm$.06  & \textbf{18.03}$\pm$0.94 & \textbf{0.46}$\pm$.05 & \textbf{0.58}$\pm$.06  \\
\bottomrule
\end{tabular}
\end{table}
Furthermore, we provide additional visual comparisons to illustrate the reconstruction performance. \Cref{fig: QualitativeResultsFFHQ_supp} displays more reconstruction examples on the FFHQ dataset under the compression ratio $n/m=16$ and noise level $\sigma=0.5$. Compared to the evaluated baseline methods, PnP-OneBit recovers facial details with fewer visual artifacts. In Addition, \Cref{fig: QualitativeResultsImageNet_supp} compares PnP-OneBit with Diff-OneBit on the ImageNet dataset across varying noise levels ($\sigma=0.5$ and $\sigma=1.0$) at $n/m=16$.

\begin{figure}[htbp]
    \centering
    \includegraphics[width=1.0\linewidth]{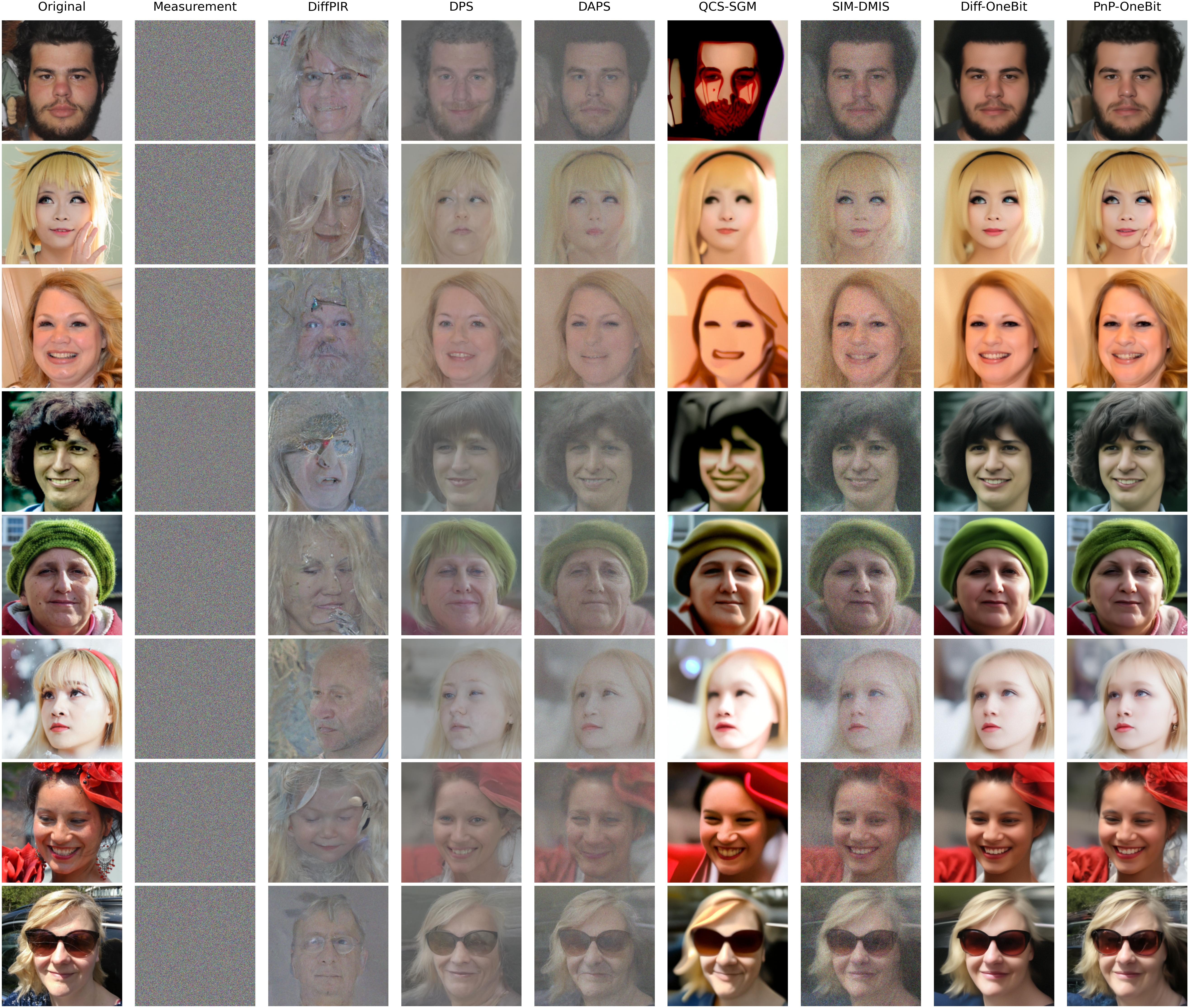}
    \caption{Qualitative comparison of different methods on the FFHQ dataset with compression ratio $n/m=16$ and additive Gaussian noise $\sigma=0.5$.}
    \label{fig: QualitativeResultsFFHQ_supp}
\end{figure}

\begin{figure}[htbp]
    \centering
    \includegraphics[width=0.8\linewidth]{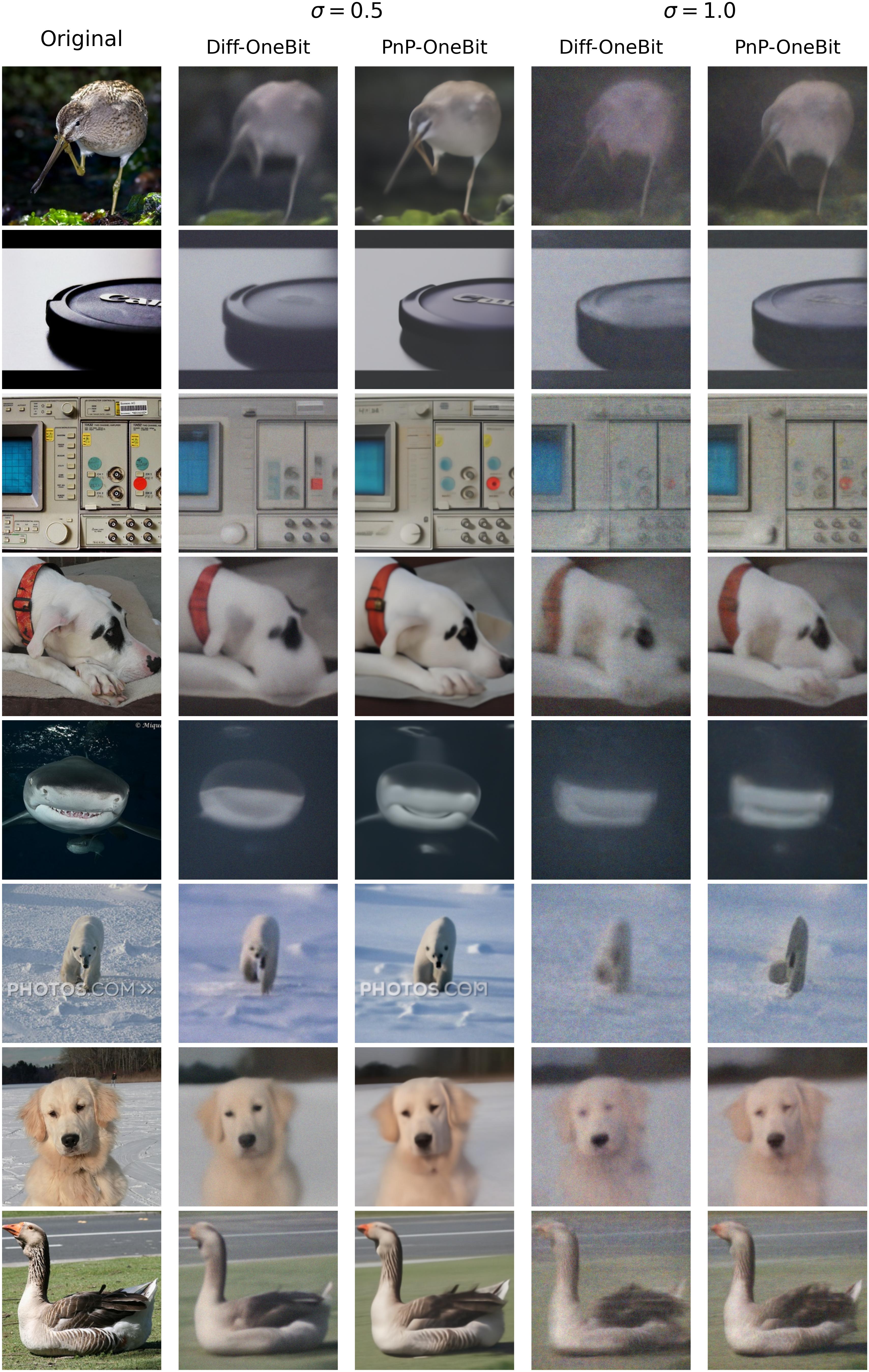}
    \caption{Qualitative comparison between Diff-OneBit and PnP-OneBit on the ImageNet dataset with compression ratio $n/m=16$ under different noise levels ($\sigma=0.5$ and $\sigma=1.0$).}
    \label{fig: QualitativeResultsImageNet_supp}
\end{figure}

% \subsection{Ablation study}
% \label{appendix: AblationStudy}

% \begin{table}[htbp]
% \centering
% \caption{Ablation study on the number of $\hat\alpha$ on the FFHQ dataset with $n/m=16$ and the additive Gaussian noise of $\sigma=0.5$.}
% \label{tab: AblationIterations4HatAlpha}
% \begin{tabular}{llllcccc}
% \toprule
% $\hat\alpha$ & $\varrho_0$ & $\nu$ & $\varrho_{\min}$ & PSNR$\uparrow$ & SSIM$\uparrow$ & LPIPS$\downarrow$ \\ 
% \midrule
% 0.0  & 1   & 0.98  & 0.3  & 18.96 & 0.46 & 0.54 \\
% 0.5  & 10  & 0.9   & 0.3  & 22.37 & 0.61 & 0.39 \\
% \rowcolor{gray!20}0.75 & 10  & 0.9   & 0.3  & 22.96 & 0.69 & 0.33 \\
% 1.0  & 10  & 0.9   & 0.3  & 21.74 & 0.56 & 0.45 \\
% \bottomrule
% \end{tabular}
% \end{table}

% \begin{table}[htbp]
% \centering
% \caption{Ablation study on the number of outer iterations ($K$) and inference time on the FFHQ dataset with $n/m=16$ and the additive Gaussian noise of $\sigma=0.5$.\textcolor{red}{TODO}}
% \label{tab: AblationIterationsAndTime}
% \begin{tabular}{lcccc}
% \toprule
% $K$ & Time (s/img) & PSNR$\uparrow$ & SSIM$\uparrow$ & LPIPS$\downarrow$ \\ 
% \midrule
% $20$  & XX.XX & XX.XX & 0.XX & 0.XX \\
% $50$  & XX.XX & XX.XX & 0.XX & 0.XX \\
% \rowcolor{gray!20}$100$ & XX.XX & 22.96 & 0.69 & 0.33 \\
% $200$ & XX.XX & XX.XX & 0.XX & 0.XX \\
% \bottomrule
% \end{tabular}
% \end{table}

\iffalse
\clearpage
\newpage
\input{checklist}
\fi

\end{document}